\documentclass{article}
\usepackage{amsthm}
\usepackage{PRIMEarxiv}
\usepackage[utf8]{inputenc} 
\usepackage[T1]{fontenc}    
\usepackage{hyperref}       
\usepackage{url}            
\usepackage{booktabs}       
\usepackage{amsfonts}       
\usepackage{nicefrac}       
\usepackage{microtype}      
\usepackage{lipsum}
\usepackage{fancyhdr}       
\usepackage{graphicx}       
\graphicspath{{media}}     

\usepackage{xcolor}
\usepackage{amssymb}
\usepackage{amsmath}
\theoremstyle{definition}
\usepackage{graphicx}

\newcommand{\true}{\top}
\newcommand{\false}{\bot}
\newcommand{\dia}[1]{\ensuremath{\langle #1 \rangle}}
\newcommand{\abs}[1]{\ensuremath{\pipe #1 \pipe}}
\newcommand{\upd}[1]{\ensuremath{\dia{U,#1}}}
\newcommand{\updk}{\ensuremath{\dia{(U_1,u_1) \cup \dots \cup (U_k,u_k)}}}
\newcommand{\delprod}{\ensuremath{DEL_{\otimes}}}
\newcommand{\delstar}{\ensuremath{DEL_{\ast}}}

\newcommand{\delminusprod}{\ensuremath{DEL^{-}_{\otimes}}}
\newcommand{\delminustar}{\ensuremath{DEL^{-}_{\ast}}}
\newcommand{\obs}[1]{\ensuremath{O_{#1}}}
\newcommand{\new}[1]{\ensuremath{O^{+}_{#1}}}
\newcommand{\del}[1]{\ensuremath{O^{-}_{#1}}}

\newcommand{\pipe}{\text{\textbar}}

\newtheorem{definition}{Definition}
\newtheorem{theorem}{Theorem}

\newtheorem{example}{Example}%
\newtheorem{remark}{Remark}%

\usepackage[natbibapa,nodoi]{apacite}

\title{Changing agents and ascribing beliefs in dynamic epistemic logic}

\author{
  Shikha Singh\\
  Computer Science and Engineering\\
  IIT Madras,
  Chennai 600036\\
  \texttt{cs16d008@smail.iitm.ac.in} \\
   \And
  Kamal Lodaya \\
  Bengaluru 560064\\
   \AND
   Deepak Khemani \\
   Computer Science and Engineering\\
  IIT Madras, Chennai 60036\\
  \texttt{khemani@iitm.ac.in} \\
}

\begin{document}
\maketitle

\begin{abstract}
In dynamic epistemic logic \citep{van2007dynamic} it is customary to use
an \emph{action frame} \citep{baltag1998logic, baltag2004logics} to describe
different views of a single action.
In this article, action frames are extended to add or remove agents, we call 
these \emph{agent-update frames}.
This can be done selectively so that only some specified agents get 
information of the update,
which can be used to model several interesting examples
such as private update and deception, studied earlier by \citet{sakama, 
baltag2004logics, Van2012lying}.
The \emph{product update} of a Kripke model by an action frame
is an abbreviated way of describing the transformed Kripke model which is the
result of performing the action.
This is substantially extended to a \emph{sum-product update} of a Kripke model
by an agent-update frame in the new setting.
These ideas are applied to an AI problem of modelling a story.
We show that dynamic epistemic logics, 
with update modalities now based on agent-update frames, 
continue to have sound and complete proof systems.
Decision procedures for model checking and satisfiability have expected complexity.
For a sublanguage, there are polynomial space algorithms.
\end{abstract}

\begin{quote}
\textit{O, what a tangled web we weave when first we practise to deceive!} \cite{walter}
\end{quote}


\section{Introduction}\label{sec:introduction}

Epistemic and doxastic logics have been used to model many information and communication situations in multi-agent domains \citep{fagin2004reasoning, meyer2004epistemic}. The language is defined on a fixed set of propositional and agent variables and its semantics is given in terms of Kripke models, also known as possible worlds semantics \citep{kripkemodal,hintikka1962knowledge}. 
Dynamic epistemic logic is used to model knowledge acquisition and belief revision of agents \citep{van2007dynamic,baltag1998logic, baltag2004logics}. 
Updates in an action framework can represent private as well as public announcements \citep{plazapublic}.
In artificial intelligence, $DEL$ has been used to model forward progression in epistemic planning \citep{bolander2017gentle, baral2022action}.
Extensions of $DEL$ were proposed to formalise false-belief tasks \citep{bolander2018seeing}.

In this article we consider an AI application \citep{singh2019planning} of modelling a story in which agents can be added and deleted.
The agent updates can be done selectively, 
leading to an informational update for only a subset of agents. 
The set of propositions remains fixed.

With nested modalities it is natural for an agent to have beliefs about other agents, including their existence.
A key question is where do the beliefs of a new agent come from.
Things get really exciting when fictional agents come into the picture,
where their creators themselves do not believe in their agency, but other agents swallow the fiction,
like Captain William Martin in Operation Mincemeat conjured up by the Allies in World War II, as in the film of \cite{mincemeat}.
In this military situation, the beliefs were planted by the Allies.
But in a more routine situation,
such as when new members join a group on social media or follow a person's page,
what beliefs can be ascribed to them?


\cite{baltag1998logic} introduced a product update which takes a model before the update, executes an action, and gives a new model after the update.
They gave several examples of its use and an axiomatization reducing $DEL$ to epistemic logic EL (without updates).
Satisfiability was also reduced to the simpler logic.
The complexity of satisfiability was determined by \cite{aucher2013complexity}.

In this article, the product update operation undergoes a substantial enhancement.
It incorporates a sum to accommodate beliefs of new agents
and a difference to remove beliefs of deleted agents.
Axiomatization and completeness follow $DEL$ techniques.
Complexity of satisfiability is preserved.
A sublanguage with better complexity is identified.

We show the application of the new framework to modelling a children's story, \emph{The Gruffalo} by  \cite{donaldson1999gruffalo}.

Here is an outline of the article.

Section~\ref{sec:background} gives some basic definitions of dynamic epistemic logics, 
as well as the action frames which our work extends.
In Section~\ref{sec:update-models}, we describe our new agent-update frames and show agent-updates with several examples.
These show the expressiveness available with our general mechanism of sum-product update.
Section~\ref{sec:detailed-gruffalo} applies our definitions to follow 
the children's story,
which we found explores complex ideas that required ingenuity to deal with technically.
We view the ability to model this story as a contribution of the formalism.
Section~\ref{sec:aul} generalises the logics of Section~\ref{sec:background} to our new updates, 
and has more traditional theoretical results.
We discuss a proof system to derive valid formulas. 
We prove its soundness and completeness, and argue that the complexity of model checking the logic as well as deciding validity of formulas does not 
exceed that of dynamic epistemic logic.
Finally, Section~\ref{sec:lit} discusses related work.


\section{Background}\label{sec:background}
Let $A$ be a nonempty set of agents.
A language $Prop(A)$ of fluents over $A$ (mostly we write it as $Prop$) is formed from
a signature of predicate symbols $p$, each with a fixed nonnegative arity. 
Thus $p(i,j)$ is a fluent for a binary predicate symbol, instantiated to agents $i$ and $j$.  

A constant domain Kripke model \citep{kripkemodal} over agents $A$ and fluents $Prop$ can be defined as follows.
\begin{definition}[Kripke model]\label{def:kripke-model}
$M=(S,\{R_i \mid i \in A\},I)$, where 
$S$ is a set of possible worlds, 
$R_i$ $\subseteq S \times S$ are accessibility relations for every agent $i \in A$, and 
$I: Prop \to 2^S$ is an interpretation function assigning states where the fluent is $true$. 
$sR_it$ abbreviates $\langle s,t \rangle\in R_i$, it means that at a world $s$, agent $i$ believes possible that the world may be $t$. 
We use pointed Kripke models, written as $(M,s)$ where $s \in S$ is a \emph{designated} state. 
\end{definition}

In the figures, a directed arrow labelled with $i$ from world $s$ to world $t$ depicts $sR_it$ 
and a bidirectional arrow between two worlds, say $s$ and $t$, labelled with $i$, represents arrows for $sR_it$ and $tR_is$. 

When used as an input to an algorithm, the size of a Kripke model 
is the sum of the number of states $\abs{S}$, the number of agents $\abs{A}$,  
the sizes of the accessibility relation $\abs{R_i}$ of every agent $i$ and 
the size of the interpretation, presented in some convenient manner for every fluent.
The asymptotically important component will be the sizes of the accessibility relations, which can be quadratic in the number of states.
The size of the interpretation can be large,
but since the fluents are fixed, we ignore it.
Thus the input is of size $O(\abs{A} \abs{S}^2)$.

The multi-agent scenarios that motivate us will not have all the agents believe in the agency of all the agents at all possible worlds. We consider private and deceptive updates which bring about an asymmetric information change (Section~\ref{sec:update-models}).
We build our agent-updates into a more general framework of action frames as defined by \cite{baltag1998logic} and developed by \cite{bek2006change}.\footnote{Action frames are called ``action models'' in the literature. We prefer frames to distinguish them from ``Kripke models''.}

\begin{definition}[Action frame]\label{def:actionmodel}
Let EL be a logical language defined on 
a finite set of agents $A$ and a finite set of fluents $Prop$. 
We will define this language below in Definition~\ref{def:logic}.

An \emph{action frame}
$U = (E, \{\obs{i} \mid i \in A\}, pre, post)$ 
consists of a finite set of \emph{events} $E$ 
and \emph{observability} relations for each agent: $\obs{i}$ $\subseteq E \times E$, 
together with functions $pre: E \to$ EL 
and $post: E \to (Prop \to \{\true,\false,no\})$,
which assign a \emph{precondition} and a \emph{postcondition} for each event.
\end{definition}

An action frame is the representation of an action as seen by different agents. 
A pointed action frame $(U,u)$ with $u \in E$ specifies the semantics of an action which updates a Kripke model,
applied at event $u$ where precondition $pre(u)$ holds.


The size of an action frame $U$ adds the sizes of the sets $A$ and $E$, the observability relations, 
and the preconditions and postconditions at every event. 
It is asymptotically dominated by the sizes of its observability relations, $O(|A||E|^2)$.

In figures of action frames, a directed arrow labelled with $i$ from an event $u$ to event $v$ depicts $u\obs{i}v$, it means that when event $u$ occurs agent $i$ considers event $v$ occurring. 
Where required, the precondition and postcondition of an event are shown alongside. 
Otherwise the event $u$ is a \emph{skip}, its precondition can be taken as $pre(u)=\top$ 
and the postcondition can be taken as no change, $post(u)(p)=no$ for all $p \in Prop$.

The updated model after an action is formalised as a product of a Kripke model 
with an action frame 
\citep{baltag1998logic, bek2006change}. 

\begin{definition}[Product update]
\label{def:product-update}
Given a pointed Kripke model $(M, s)$ 
and a pointed action model $(U,u)$, the resulting pointed Kripke model after the execution of the action, 
$M \otimes U$ is defined as $ (M',(s,u))$ with 
$M' = (S' , R', I')$, where 
\begin{itemize}
    \item $S' = \{(s,u) \in (S \times E) \mid ~ (M,s) \models pre(u)\}$
    \item $I'(p) = \{(s,u) \in S' \mid (s \in I(p) \mbox{~and~} post(u)(p)=no), \mbox{~or~} post(u)(p)=\top\}$
    \item $R'_a = \{((s,u),(t,v)) \mid {s}R_a{t},~ u\obs{a}v\}$, for $a$ in $A$
\end{itemize}
\end{definition}

The \emph{ontic} effect of the action is to \emph{set to true} those fluents where the postcondition $post(u)(p)$ is $\top$, \emph{set to false} those fluents where the postcondition $post(u)(p)$ is $\false$ and \emph{leave unchanged} those where it is $no$.

The number of states in the updated Kripke model is the product $\abs{S} \abs{E}$ of
the number of states in the starting Kripke model and the number of events in the 
action model.

Now we define our logical languages. 

\begin{definition}[Formulas of languages EL, DEL-, DEL]
\label{def:logic}
Starting with a set of agents $A$ and fluents $Prop$, the formulas of the logical language DEL are constructed using the following BNF. 
\[\phi ::= \top \mid p \in Prop \mid \neg \phi \mid  
(\phi \wedge\phi) \mid P_i\phi,~i \in A \mid \upd{u}\phi
\mid \updk\phi\] 
The sublanguage DEL- does not have the union update formulas $\updk\phi$.
EL is a sublanguage of DEL- which does not have update formulas $\upd{u}\phi$.
\end{definition}

The modality $P_i\phi$ is read as `agent $i$ possibly believes $\phi$'. The dual modality $B_i\phi = \neg P_i \neg \phi$  is read as 
`agent $i$ believes $\phi$'.
$\upd{u} \phi$ is read as `possibly after update $U$, $\phi$ holds'. 
The dual modality is $[U,u]\phi$.
The length of a formula is defined as usual from its parse tree.
Note that $(U,u)$ is a representation of a pointed action frame, inside a formula.  
The size of that action frame is included in the length of the formula. 

\begin{definition}[Truth at a world in a model] 
\label{def:truth}
Given a formula $\phi$ in language DEL, and a pointed Kripke model $(M,s)$, the assertion that formula $\phi$ is true at world $s$ in model $M$ is abbreviated as $M,s \models \phi$ and recursively defined as:
\begin{itemize}
    \item $(M,s) \models \top$,
    \item $(M,s) \models p$  $\Leftrightarrow$ $ s \in I(p)$, 
    \item $(M,s) \models \neg \phi $ $\Leftrightarrow$  not $ (M,s) \models \phi$, 
    \item $(M,s)  \models (\phi \wedge \psi)$ $\Leftrightarrow$  $(M,s) \models \phi$ and $ (M,s) \models \psi$, and 
    \item $(M,s)  \models P_i \phi $ $\Leftrightarrow$  for some $t$, $sR_it$ and  $(M,t) \models \phi$
    \item $(M,s) \models  \upd{u}\phi $~iff~ $ (M,s) \models pre(u)$ and  $(M \otimes U, (s,u)) \models \phi$.
    \item $(M,s) \models \updk\phi$~iff~$(M,s) \models pre(u_i)$ and $(M \otimes U_i,(s,u_i)) \models \phi$, for some $i,~1 \leq i \leq k$.
\end{itemize}
A formula is \emph{valid} if it is true in all models at all worlds.
A formula is \emph{satisfiable} if it is true in some model at some world.
\end{definition}

With the semantics of the update modalities given using product update, 
the logics for the languages considered in Definition~\ref{def:logic} are called 
$EL$, $\delminusprod$ and $\delprod$, respectively.

We sometimes talk of `agency' by which we mean that
an agent $i$ exists at a possible world $s$ in model $M$ 
if and only if $(M,s) \models P_i \top$.
Thus accessibility relations will not be serial,
as is traditional in the literature.
An agent $i$ exists for another agent $j$ at a world $s$ if $i$'s agency holds at all the worlds $t$ reachable by $j$ from $s$. Formally, $B_jP_i\top$. 

We work only with \emph{transitive} relations $R_i$, hence $B_i\phi \implies B_iB_i\phi$ is a valid formula. It says that positive belief is introspective. 
In our models, $\neg B_i\phi \implies B_i \neg B_i\phi$ is not a valid formula. (As will be seen, 
our product construction does not preserve the Euclidean property required.) 
It says that negative belief is introspective. Such correspondences of valid formulas with properties of Kripke frames are described in the textbook of \cite{chellas}.

For a logic $\mathcal{L}$, the problem of checking satisfiability of its formulas is called $SAT(\mathcal{L})$.
We will also consider the \emph{model checking} problem 
$MC(\mathcal{L})$, which checks, given a pointed Kripke model $(M,s)$ and formula $\phi$ of $\mathcal{L}$, whether $\phi$ is true at world $s$ in $M$.
Figure~\ref{fig:complexityclasses1} shows the upper bounds known about the computational complexity of these problems, 
extended to the case where fluents are used.
The syntax and semantics of the logic at the source of an edge are available in the logic at the target of the arrow.

\begin{figure}[!htb]
    \centering
    \includegraphics[width=0.8\textwidth]{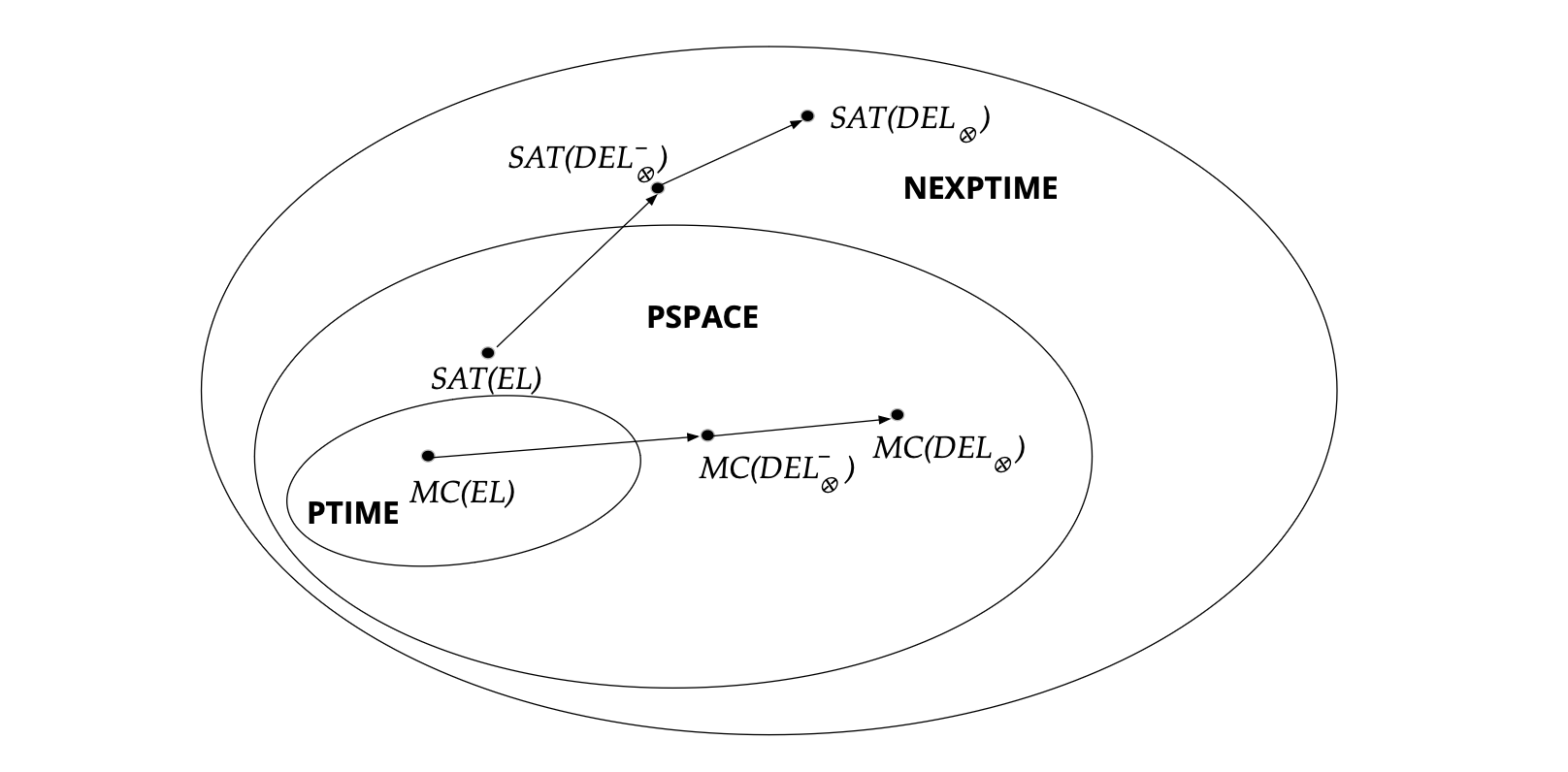}
    \caption{Known upper bounds for $MC(\mathcal{L})$ and $SAT(\mathcal{L})$}
    \label{fig:complexityclasses1}
\end{figure}


Soundness and completeness of $DEL$ with respect to an axiom system are proved in 
\cite{baltag2004logics,bek2006change}.
Decidability of the logic follows.
The complexity of checking validity is nondeterministic exponential time,
established by \cite{aucher2013complexity}.
They also establish complexity of model checking.

The product update framework does not allow one to define actions that can introduce new agents or remove agents from the Kripke models. How would one model the following?

Consider a simple dormitory setting in which the arrival of a new warden is announced to the residents. The agents in the dormitory model are the existing set of residents and staff. 

To take cognizance of the warden, a new agent after this announcement, demands a new kind of update, which is introduced 
in Section~\ref{sec:update-models}.



\section{Agent-Update Frames}\label{sec:update-models}
We formally define \emph{agent-update frames} on a countably infinite set of agents $\mathcal{A}$ and a finite $A\subseteq \mathcal{A}$.
Recall the language EL from Section~\ref{sec:background} and that the set of fluents $Prop$ is fixed.

\begin{definition}[Agent-update frame on $A\subseteq \mathcal{A}$]
\label{def:agentupdate}
An \emph{agent-update frame} is a finite 
structure $U=(E, \{\obs{i} \mid i \in A\}, \{\new{i} \mid i \in \mathcal{A}\}, 
\{\del{i} \mid i \in A\},pre,post)$ with relations for agents $i$ as indicated, 
defined on a finite set of events $E$, together with the 
precondition and postcondition functions $pre: E \to$ EL and $post: E \to (Prop \to \{\top, \false,~no\})$. 
$u\new{i}v$ means that event $u$ adds agent $i$,
we collect such added agents $i$ in the set $Add(u)$.
$u\del{i}v$ means that event $u$ deletes agent $i$,
and $Del(u)$ is the collection of such deleted agents. 
A \emph{pointed agent-update frame} is written as $(U,u)$ where $u \in E$ is a designated event.
\end{definition}

The size of an agent-update frame is as before the sum of the sizes of its components, with $3\abs{A} \abs{E}^2$ for the three kinds of relations.

In pictures, in addition to the traditional (solid) arrows (here denoted as $\obs{i}$) in
an action frame on $A$, we have two other types of arrows: \emph{sum} arrows, dashed, for $\new{i}$, which can range over new agents outside $A$, and \emph{difference} arrows, dotted, for $\del{i}$ on $A$ in the agent-update frames.

We use letters $a,b,i,j,k$ to denote agents, $s,t$ to denote worlds in Kripke frames, and $u,v,w,x$ to denote events in the agent-update frames throughout this article. 

\subsection{Examples}

We start with examples from a dormitory and from a story by \cite{donaldson1999gruffalo}. 

\begin{figure}[!htb]
    \centering
    \includegraphics[width=0.4\textwidth]{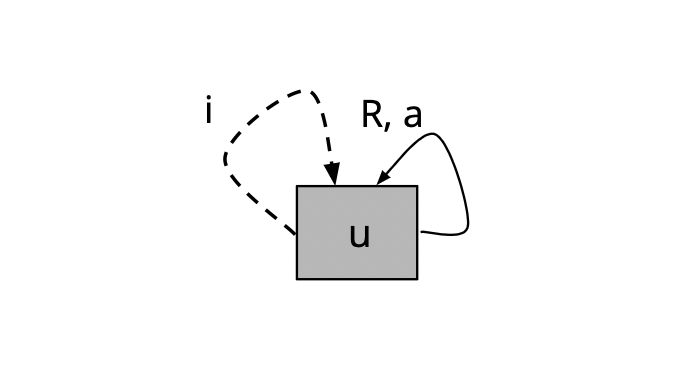}
    \caption{$i$-addition for $R \cup \{a\}$}
    \label{fig:addupdatemodel}
\end{figure}

\begin{example}\label{eg:warden}
Consider the agent-update frame illustrated in Figure~\ref{fig:addupdatemodel} in which event $u$ adds a warden agent $i$, this is shown using a dashed arrow.
The other (solid) arrows define observability of existing agents as usual. 
Let $a$ be some agent, say a security guard, and consider agents in set $R$ as residents of the dorm.
All the agents in $R \cup \{a\}$ observe the event $u$ that adds $i$ in the frame.
\end{example}

\begin{figure}[!htb]
    \centering
    \includegraphics[width=0.7\textwidth]{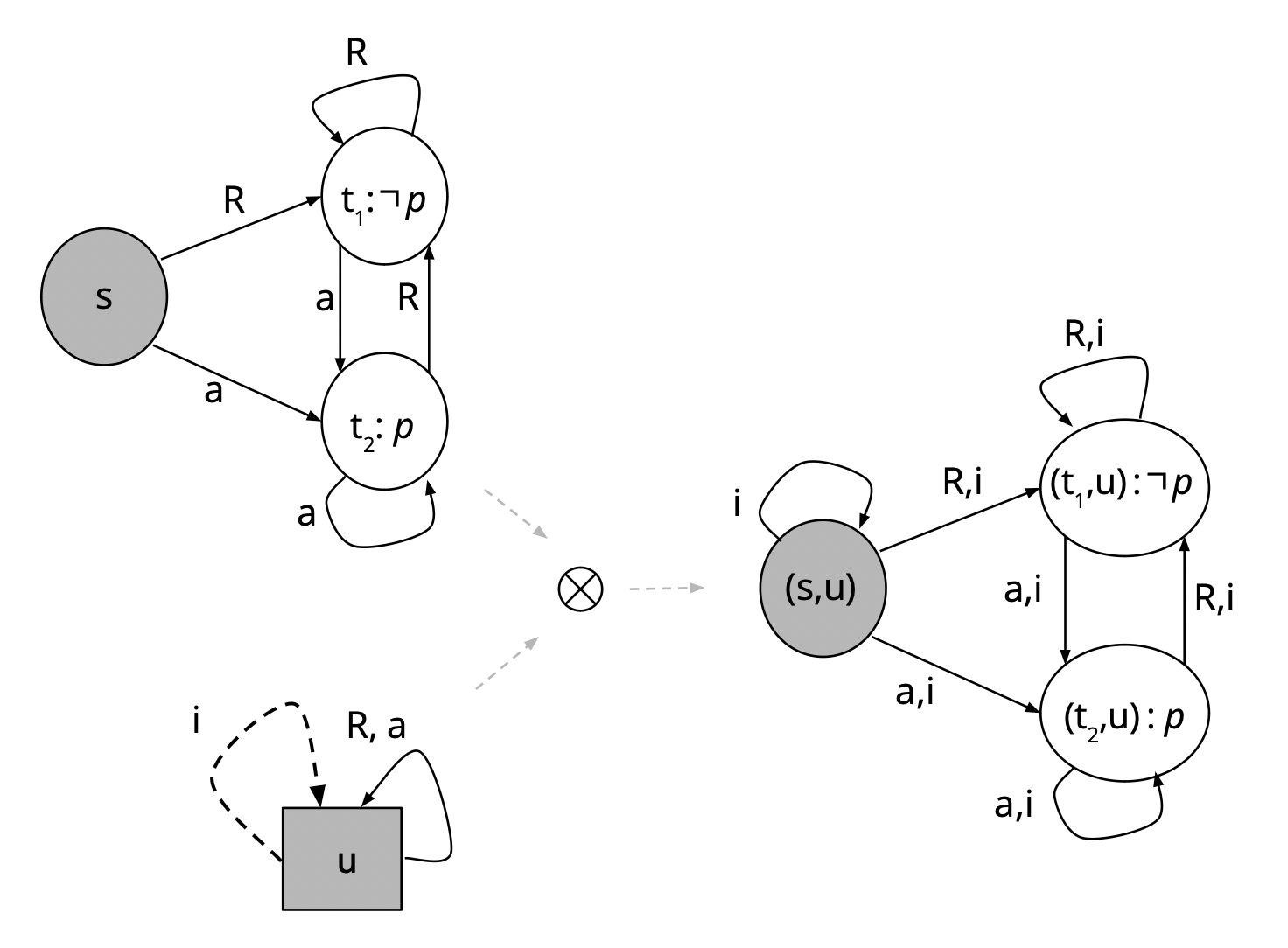}
    \caption{Agency of warden $i$ gets introduced to residents $R$ and guard $a$}
    \label{fig:successwardensumproduct}
\end{figure}

Suppose the initial situation in the dorm is depicted by the Kripke model in Figure~\ref{fig:successwardensumproduct} where at state $s$, the guard $a$ believes the fluent $p$, but the residents $R$ do not. 
What beliefs should $i$ be ascribed with after event $u$ is executed in state $s$?

Our approach is to make the warden inherit the beliefs commonly held by the residents and the guard,
except where their beliefs differ from the warden's. 
One reason is that group beliefs about the dorm, such as the existence of its residents, are automatically communicated to the warden.
Think of all these beliefs being listed on the group's webpage.
The issue of beliefs implicitly held by a group is more tricky, we do not consider it.

The effect of the agent-update in Example~\ref{eg:warden} is depicted in Figure~\ref{fig:successwardensumproduct}. 
At $(s,u)$, agent $i$ inherited arrows of $R$ and $a$ from world $s$ to $t_1$ and $t_2$, respectively. Therefore, the warden does not inherit any belief about $p$.

\begin{example}\label{eg:noneuclidean}
Modifying the earlier example, suppose there is no guard. 
The warden, believing $p$ in the initial structure, is introduced to the residents $R$, none of whom believes $p$.
The effect of the update will be that the warden's belief in $p$ is weakened, 
conceding that it is possible to believe $p$ as well as $\lnot p$.
\end{example}


\begin{figure}[!htb]
    \centering
    \includegraphics[width=0.4\textwidth]{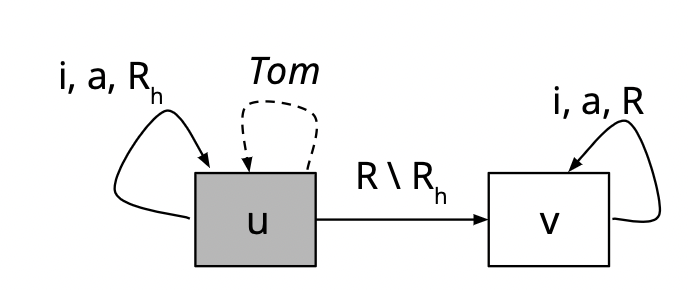}
    \caption{$Tom$-addition for $R_h \cup \{i,a\}$}
    \label{fig:privateaddupdatemodel}
\end{figure}

\begin{example}\label{eg:Tom}
Suppose a new resident, say $Tom$, joins a hall $h$ in the dorm. The residents of the hall $R_{h}$ witness his arrival, as do the warden $i$ and the security guard $a$, as indicated by the dashed arrow. Other residents $R\setminus R_{h}$ are unaware of the agency of $Tom$ at this moment. Such a private agent update is illustrated in Figure~\ref{fig:privateaddupdatemodel}. 
Agents in $R_h$, $i$ and $a$ see $Tom$ being added at $u$. 
Agents in $R\setminus R_{h}$ are oblivious, they believe that beliefs of $R$, $i$ and $a$ are unchanged.
\end{example}

\begin{figure}[!htb]
    \centering
    \includegraphics[width=0.9\textwidth]{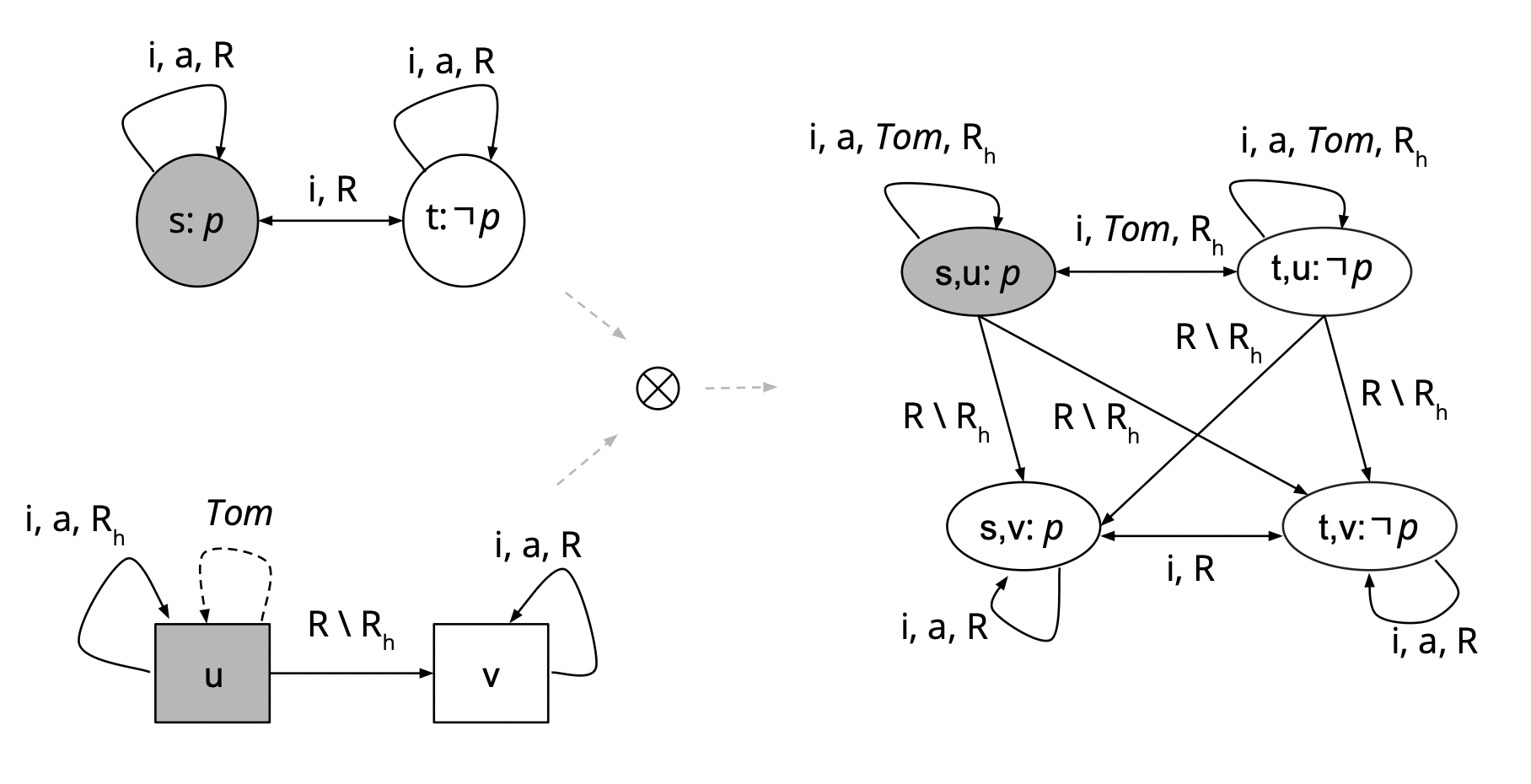}
    \caption{$Tom$ joins the dorm, only residents in $R_h$ and $\{i,a\}$ observe the event}
    \label{fig:successtimsumproduct}
\end{figure}

In Example~\ref{eg:Tom}, only warden $i$, guard $a$ and residents $R_h$ become aware of $Tom$ joining the dorm (event $u$), therefore $Tom$ inherits the possible worlds from $i$, $a$ and $R_h$ from the input model into the updated model as shown in  Figure~\ref{fig:successtimsumproduct}. 
Beliefs of $R \setminus R_h$ remain static (do not change after the update), while beliefs of $R_h \cup \{i,a\}$ remain static except for the beliefs about the agent being added: $Tom$.

\begin{figure}[!htb]
    \centering
    \includegraphics[width=0.4\textwidth]{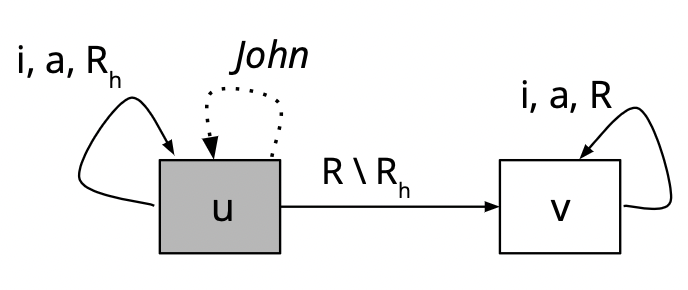}
    \caption{$John$-removal for $R_h \cup \{i,a\}$}
    \label{fig:privatedelupdatemodel}
\end{figure}

\begin{figure}[!htb]
    \centering
    \includegraphics[width=0.8\textwidth]{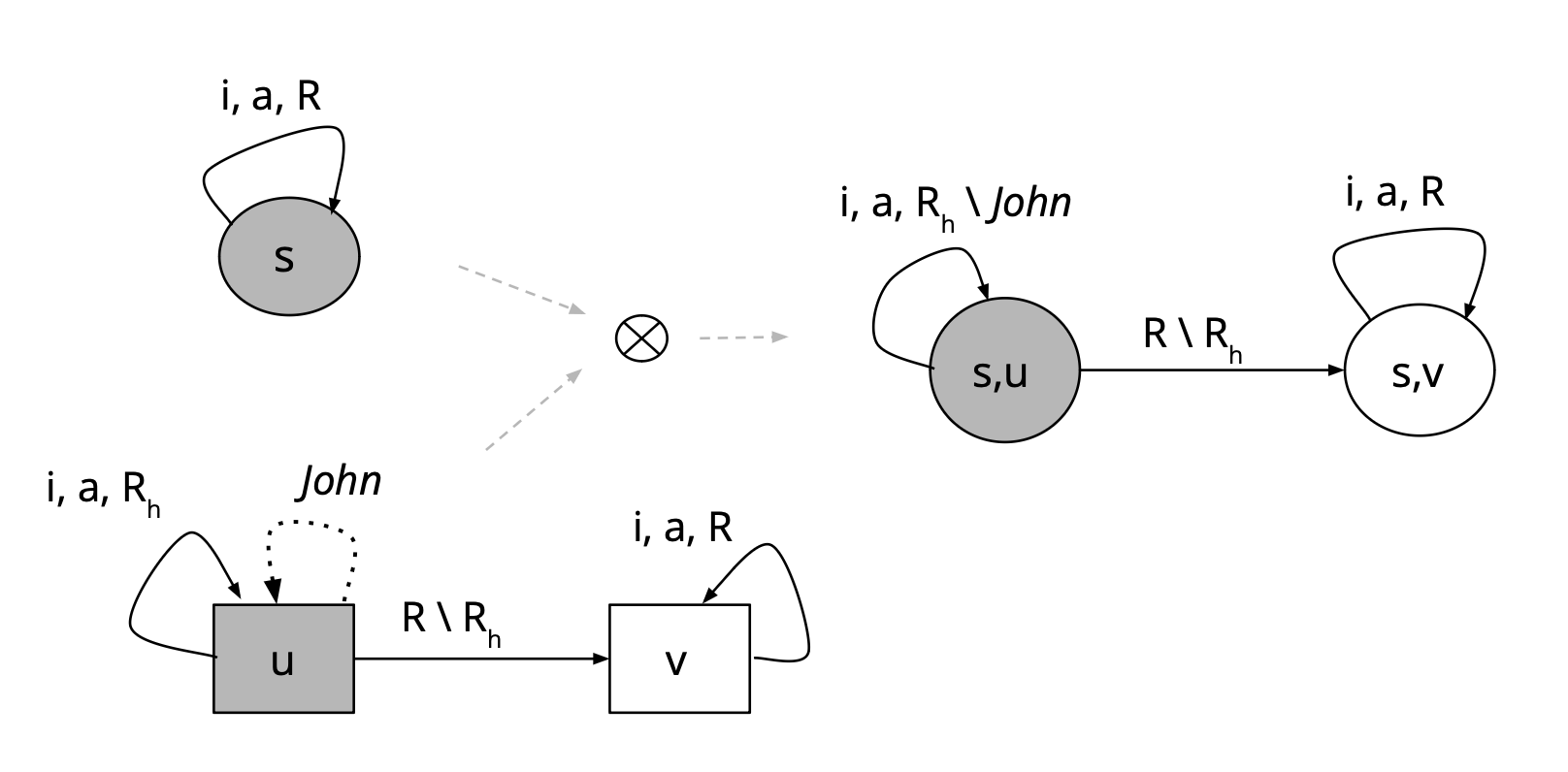}
    \caption{$John$'s departure observed by residents in $R_h \cup \{i,a\}$}
    \label{fig:successjansumproduct}
\end{figure}

\begin{example}\label{eg:John}
Similarly a private deletion of an agent happens when one of the residents, say $John$, leaves the dorm, and the departure is observed only by a subset of residents $R_h$ and $\{i,a\}$, the rest being oblivious of the departure. In Figure~\ref{fig:privatedelupdatemodel}, the dotted self-loop for $John$ at $u$ denotes that resident $John$ is to be deleted. The agents in $R_h \cup \{i,a\}$  observe this event at $u$. Beliefs of rest of the residents in $R \setminus R_h$ do not change. 
In particular, they continue to believe that $R \cup \{i, a\}$ believe in the agency of $John$ at $v$.
Model update is shown in  Figure~\ref{fig:successjansumproduct}. 
\end{example}

\begin{figure}[!htb]
    \centering
    \includegraphics[width=0.45\textwidth]{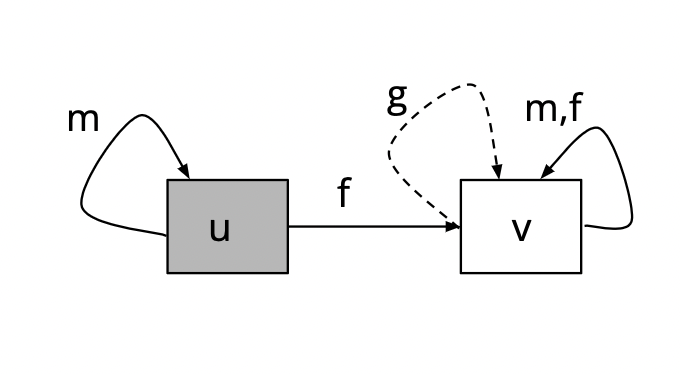}
    \caption{Mouse $m$ deceives fox $f$ about a gruffalo $g$}
    \label{fig:dcvaddupdatemodel}
\end{figure}

\subsection{A children's story}
\label{sec:story}
To introduce deception, we look at \emph{The Gruffalo} \citep{donaldson1999gruffalo}. 
This children's story has a clever mouse who invents a fictitious creature called gruffalo, in order to scare some predators away. 
The fox, the owl and the snake in turn are duly convinced into fleeing as this creature likes to eat the others. 
But then the story takes a curious turn when the mouse runs into the gruffalo in flesh and blood. 
The clever mouse finds a way to its safety in a different manner this time. 
It boasts to the gruffalo that the fox, the owl and the snake are scared of it as the mouse is the most dangerous creature in the jungle. 
It takes the gruffalo to the fox, which flees on seeing the gruffalo in flesh and blood. 
The gruffalo understands this to be because the fox is afraid of the dreaded mouse.
After this episode is repeated with the owl and the snake too, the gruffalo itself flees.

\begin{example}\label{eg:gruffalo}
Two events in the agent-update frame shown in Figure~\ref{fig:dcvaddupdatemodel} depict
the mouse's action of misleading the fox into believing a fictitious gruffalo: $u$ as perceived by the mouse $m$ and $v$ as perceived by the fox $f$. 
The presence of a dashed self-loop on event $v$ means that agent $g$ (the gruffalo) is added by $v$. 
\end{example}

\begin{figure}[!htb]
    \centering
    \includegraphics[width=0.9\textwidth]{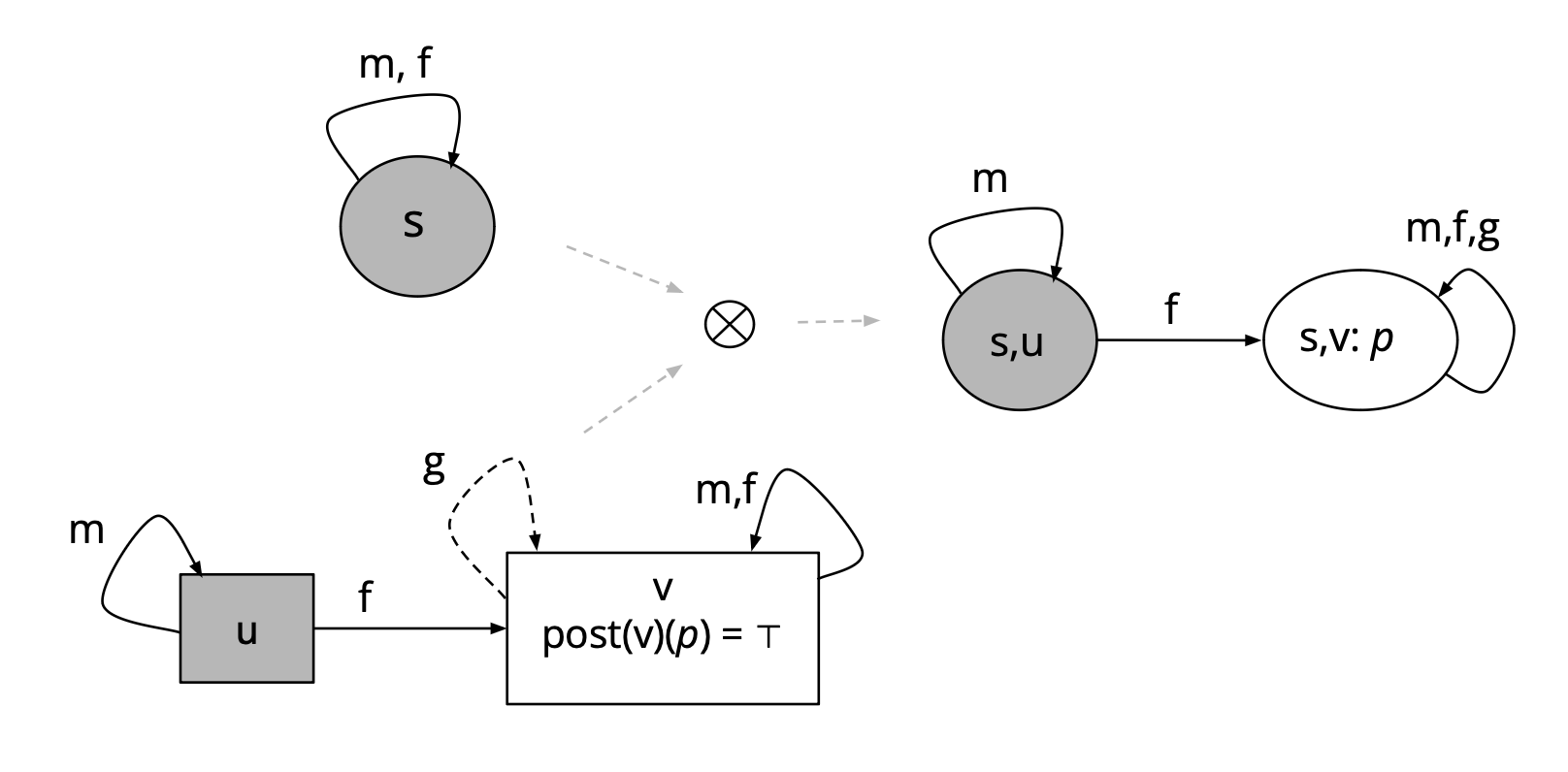}
    \caption{Deceptive update about gruffalo's agency}
    \label{fig:successgruffalosumproduct}
\end{figure}

In Example~\ref{eg:gruffalo} the fox $f$ is deceived into believing in the agency of gruffalo $g$ after mouse $m$'s announcement, while $m$ itself does not believe in $g$'s agency, as shown in Figure~\ref{fig:successgruffalosumproduct}.
In the fox's imagination, gruffalo believes the proposition \emph{p} (gruffalo eats fox), 
 a belief induced by $m$. 
So a new belief is ascribed to gruffalo which is not inherited.

Similarly, we can have a deceptive deletion of an agent, which we leave to the reader.
Section~\ref{sec:detailed-gruffalo} looks at the story in more detail, showing how beliefs are ascribed to new agents in more complex scenarios. 

In our doxastic models, agency at a state $s$ is about the belief that an agent exists.
Thus whether an existing agent $Tom$ arrives at a scene,
or a fictitious agent $g$ `arrives' in the imagination of an agent (at a state in the model representing beliefs of the agent), both are treated in the same way. 
The distinction is the following:
$Tom$'s arrival is witnessed by the agents present,
and their beliefs regarding his presence change;
the residents in the other hall do not witness his arrival, and their beliefs regarding his presence do not change.
Belief in the gruffalo $g$ enters the mind of the fox $f$, and there is the mouse $m$ who `witnesses' this change of belief, yet its belief about the non-existence of $g$ does not change.

We identify a set of agents whose beliefs remain unchanged at an event in an agent-update frame.

\begin{definition}[Observer]\label{def:observer}
The set of \emph{observers} $Obs(u)$ at an event $u$ in an agent-update frame is those $j$ with agency at $u$ such that $u\obs{j} v \iff v=u$.
\end{definition}


\emph{Deceivers} are certain observers who remain unmoved by changes in beliefs of observed agents,
we use the term informally.
In planning applications \citep{khemani},
other details such as actors (the agent performing the action), location and visibility may have to be provided \citep{singh2023two}.

\subsection{Deception down the line}
\label{sec:obsdec}

\subsubsection{Observing deception}

\begin{figure}[!htb]
    \centering
    \includegraphics[width=0.55\textwidth]{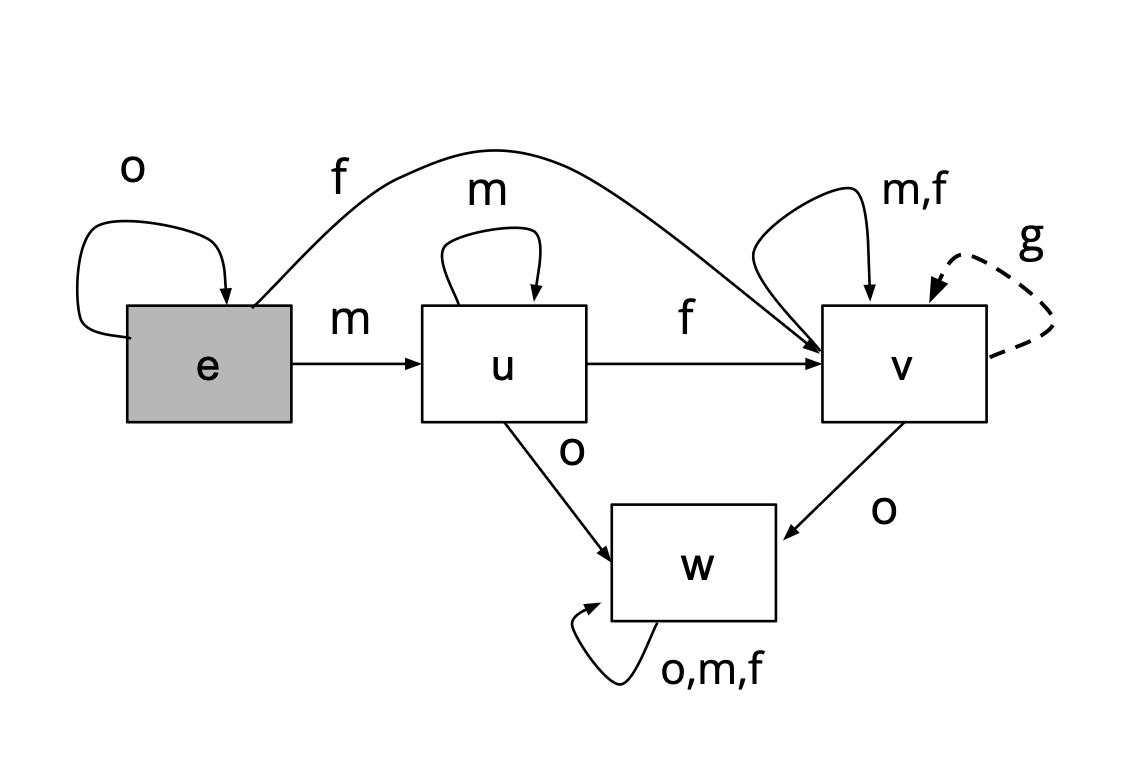}
    \caption{Owl observes mouse deceiving fox}
    \label{fig:doubledeception}
\end{figure}

Consider a wise owl $o$ observing deceiver $m$ (oblivious of the owl's presence because owl is eavesdropping) deceiving $f$ into believing in the addition of agent $g$, 
as illustrated in the agent update frame illustrated in Figure \ref{fig:doubledeception}. Event $e$ encodes the observation of deception by $o$. At $u,v$, $m$ and $f$ consider $o$ to be oblivious and therefore $o$ is considered observing a $skip$ by both. $m$ considers $u$ where it tries to deceive $f$.

\subsubsection{Simultaneous deception}

\begin{figure}[!htb]
    \centering
    \includegraphics[width=0.6\textwidth]{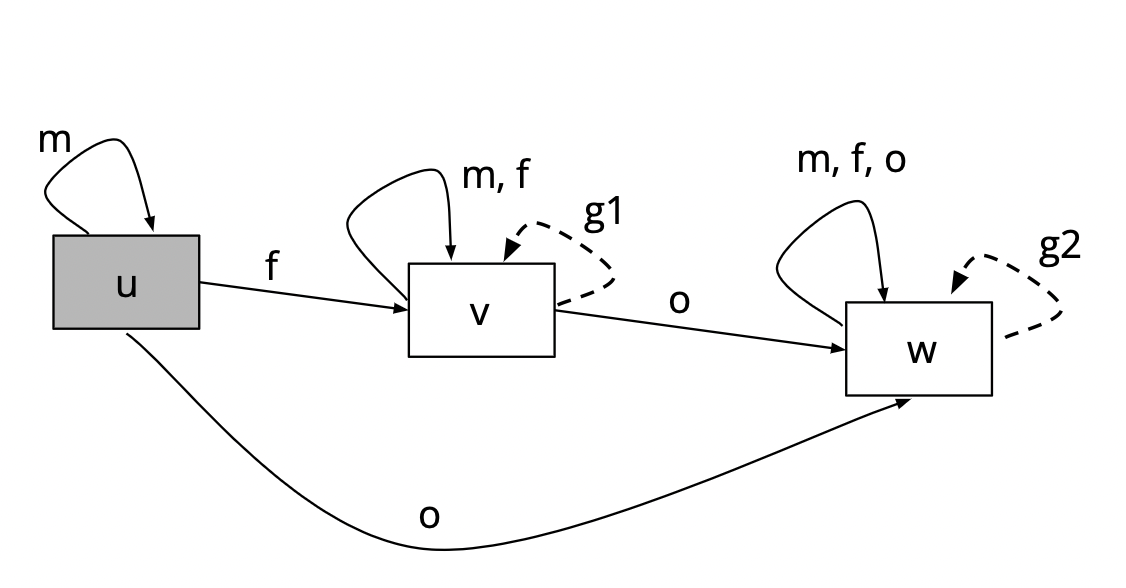}
    \caption{$f,m$ deceive $o$ while $m$ deceives $f$}
    \label{fig:simul-deception}
\end{figure}

In the agent update frame illustrated in Figure \ref{fig:simul-deception},
while $f$ is being deceived by $m$ into believing in the addition of agent $g1$, an owl $o$ is deceived by $f$ and $m$ into believing in the addition of agent $g2$.

We can also make $m$ oblivious of the deceit carried out by $f$.

By now the idea should be clear: arbitrary levels of observation and deception can be modelled.
More work is needed to better formalise this observation.

Next we put together all that we have seen so far into a single operation generalising product update.
The specific updates which we saw above 
contributed significantly to the formulation of the next definition.

\subsection{Sum-product update}

We define sum-product update to describe belief update for existing agents, to ascribe beliefs to newly added agents, and to drop beliefs of deleted agents. 
During model transformation, for an existing agent $a$, possible worlds for an agent in the updated model include the (unforgotten) worlds it considered possible earlier.
In world $(s,u)$ (after execution of event $u$ in world $s$) of the product model,  another world $(t,v)$ is possible if and only if $t$ is possible from $s$, and $v$ is possible from $u$. 
For the agent $i$ being added by
an agent-adding event $u$ ($i \in Add(u)$), the worlds that $i$ considers possible at $(s,u)$ are \emph{observer-dependent} (see item \emph{inherited} below). 

The beliefs of the existing agents are determined by product, the beliefs
of the newly added/deleted agents are determined by sum/difference. 
We describe the transformation of a model on $A$ when an agent-update frame on $\mathcal{A}$ is applied 
to it, and we call it \textit{sum-product update}.
An agent's deletion takes priority over its addition.

\begin{definition}[Sum-product update]
\label{def:sum-product}
Given a pointed Kripke model $(M,s)$ on agents $A$
and a pointed agent-update frame $(U,e)$ with $U = (E, \obs{}, \new{}, \del{}, pre, post)$ 
on agents $\mathcal{A}$, 
the updated pointed Kripke model $(M * U,(s,e))$, is defined as: 
$(S',\{R'_a \mid a \in A'\},I')$ on the updated set of agents $A'$ (those $a$ such that $R'_a$ is nonempty), where:
\begin{itemize}
    \item $S' = \{ (s,u) \in S \times E\mid M,s \models pre(u) \}$
    \item $I'(p) = \{(s,u) \in S' \mid (s \in I(p) \mbox{~and~} post(u)(p)=no), \mbox{~or~} post(u)(p)=\true\}$
    \item $R'_a$ is the transitive closure of 
    $(Q_a^{unf} \cup Q_a^{asc} \cup Q_a^{inh})$, where:
    \begin{description}
    \item[unforgotten:] $(s,u)Q_a^{unf}(t,v) \iff  
    s R_a t \mbox{~and~} u \obs{a} v \mbox{~and~not~} u \del{a} v$
    \item[ascribed:] $(s,u)Q_a^{asc}(s,v) \iff u\new{a}v,
		    \mbox{~for~} a\in (Add(u) \setminus Del(u))$
	\item[inherited:] $(s,u)Q_a^{inh}(t,u) \iff s R_{Obs(u)} t, 
		    \mbox{~for~} a\in (Add(u) \setminus Del(u))$
    \end{description}
\end{itemize}
\end{definition}

\noindent
The updated Kripke model is again constant domain,
but over the new set of agents $A'$.
The number of states in it is the product $\abs{S} \abs{E}$ of the number of states in the starting Kripke model and the number of events in the update frame.
Its size will be asymptotically dominated by the size of its relations $3\abs{A'} \abs{S}^2 \abs{E}^2$.

The accessibility relation for an agent, say $a$ in the transformed model, is defined such that after execution of an event $u$ in the world $s$, it considers a world $(t,v)$ possible at world $(s,u)$ under the following conditions, which elaborate on the formal definition above.
\begin{enumerate}
    \item If $a$ is neither being added nor deleted, this is the product of arrows of Definition~\ref{def:product-update}.
    We call the beliefs accessible in this way \emph{unforgotten}.
    \item If $a$ is being deleted (which takes priority over addition,
    then the beliefs accessible by undeleted $a$-arrows are also called \emph{unforgotten}.
    \item If $u$ is an $a$-adding event and not simultaneously
    an $a$-deleting event, there are three cases.

\begin{enumerate}
    \item If $a$-arrows were already present, they are handled using product. These are again \emph{unforgotten} beliefs.
    \item An explicit $a$-adding arrow is from event $u$ to event $v$, both executable at $s$. New beliefs at $v$
    are \emph{ascribed} by the updating agent.
\item Worlds $t$ considered possible from world $s$ by an observer at $u$, provided $pre(u)$ holds in $s$ are \emph{inherited} from the observer. 
\end{enumerate}
\end{enumerate}

Note that inheritance is asymmetric. 
If the agent being introduced to the observers already had beliefs, its beliefs are not communicated to the observers.
This is required to model lying, see \citep{sakama2010logical,Van2012lying}.
After the update in Example \ref{eg:noneuclidean}, the Euclidean property fails,
$P_i p \land \lnot B_i P_i p$.

To summarize, an agent $a$ being introduced can get beliefs in two different ways. 
The first is a direct ascription, a technique from modal logic: 
the update links to a possible world where $P_a p$ will hold for belief $p$.
One can think of introducing the new agent with an announcement that $a$ believes $p$.
The Gruffalo story in Section \ref{sec:detailed-gruffalo} has such examples of directly introducing a new agent to an existing agent. 
The second way is a \emph{default}, a technique from artificial intelligence \citep{khemani}: 
an inheritance of the common beliefs of the group it is being introduced to. 
This is required in more anonymous examples, for example when a new user (about whom not much is known) joins a Whatsapp group.
In modal logic this will require a more complex formula.
Allowing both ideas enhances applicability of the update to a wider variety of situations.

Why should the persons to whom $a$ is being introduced to,
the observers, play a role?
Consider that one of the propositions under discussion is $q$, that grass is green.
When a gruffalo enters the story, in the model one has to answer the question whether the gruffalo believes ``common sense'' propositions such as $q$, that grass is green.
A quick mechanism is needed to resolve such pragmatic questions.
Default reasoning from AI planners is used here.


\section{The story formalised}\label{sec:detailed-gruffalo}

We have seen many one-step agent-update examples. 
In this section, 
we model the story in Section~\ref{sec:story} to formally show agency-creation of \textit{gruffalo}, 
first for the \textit{fox} and \textit{owl}, then for the \textit{mouse}, 
following the course of events presented in the story.
The story is a little more complex than the scope of our models.
It uses magical realism which is difficult to model within a propositional modal framework like ours, 
which does not distinguish between appearance and reality.
Fortunately we are able to use a trick to show how the gruffalo deceives itself.

Action frames are used to model updates and not reasoning.
Following the book (which is aimed at children), there is no suggestion that the mouse understands predator-prey theory, nor that this is a precondition for deception.
A theory of what kind of \emph{causes} 
(\citep{lin1995ramification} studies these formally) could lead to actions such as agent addition and deletion, is beyond the scope of this paper. 
Action frames do not even specify which agent(s) perform the action.
We do not use a planning domain \citep{mcdermott2000}, there is no location or time in an action frame, 
so we do not attempt to adequately model actions such as `fox runs away', leaving them to the imagination of the reader.

\subsection{The fox is deceived}
The initial situation in the story is modelled with $M_0$ with a designated world $s$ as shown in Figure~\ref{fig:mousedcvsfox}.
There are several agents in the story. We begin with three: mouse $m$, fox $f$ and owl $o$, and all three believe in each others' agency. 

\begin{figure}[!htb]
    \centering
    \includegraphics[width=0.9\textwidth]{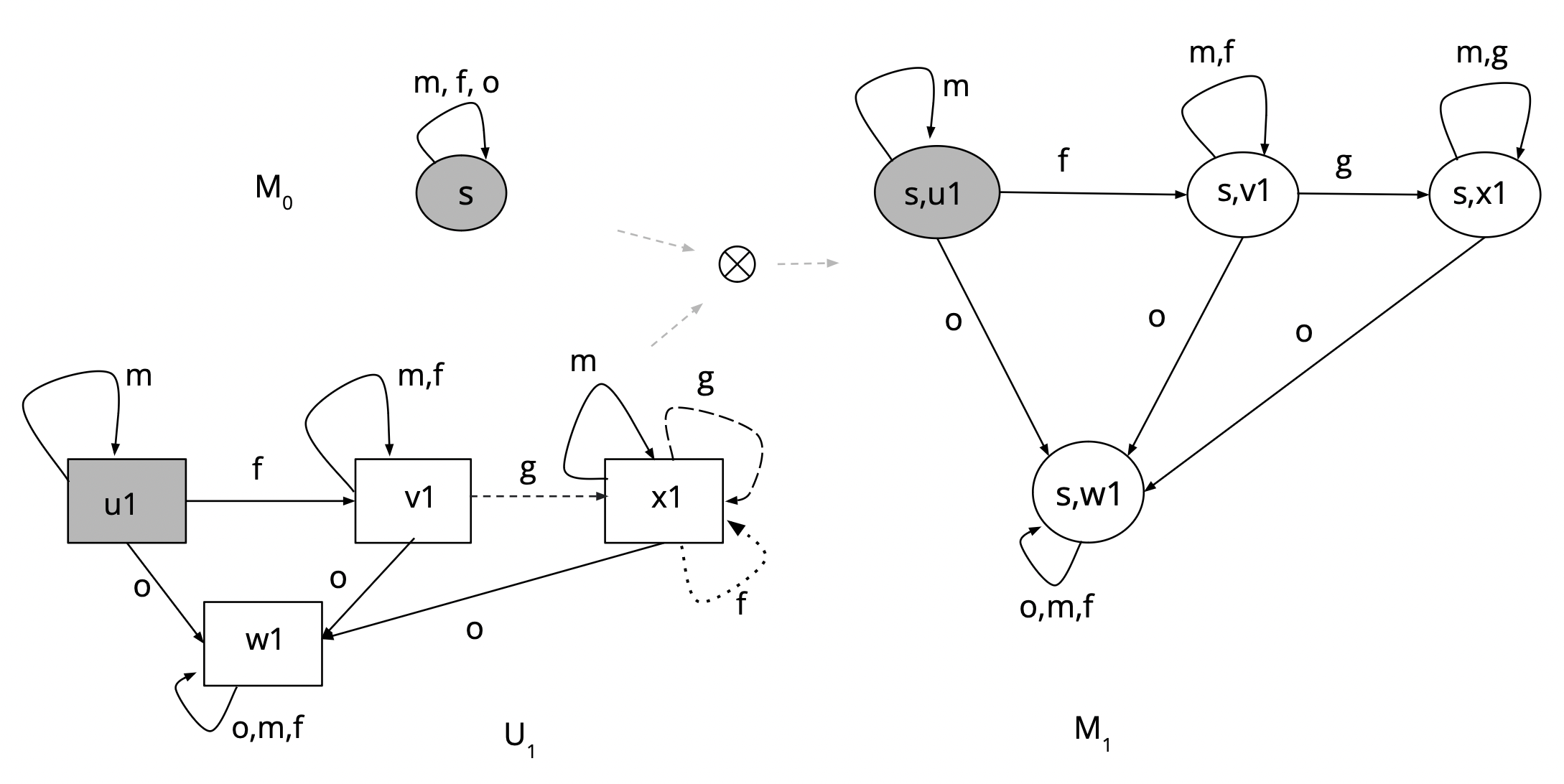}
    \caption{Mouse deceives fox that there is a fox-eating gruffalo}
    \label{fig:mousedcvsfox}
\end{figure}

\begin{example}
The first move of $m$ is a deceptive agent-addition action which introduces the agency of a gruffalo $g$ which, $m$ announces, likes to eat fox. 
(This was modelled using a postcondition in Example \ref{eg:gruffalo}.)
In Figure~\ref{fig:mousedcvsfox} this move is modelled as a 
combination of an addition and a deletion event,
at $u1$ the mouse deceives the fox into believing in the agency of the gruffalo; at $v1$, the fox believes the gruffalo believes in eating foxes, which we represent as a deletion of fox at $x1$.
The owl at $w1$ is oblivious of this interaction.
\end{example}

\begin{figure}[!htb]
    \centering
    \includegraphics[width=0.9\textwidth]{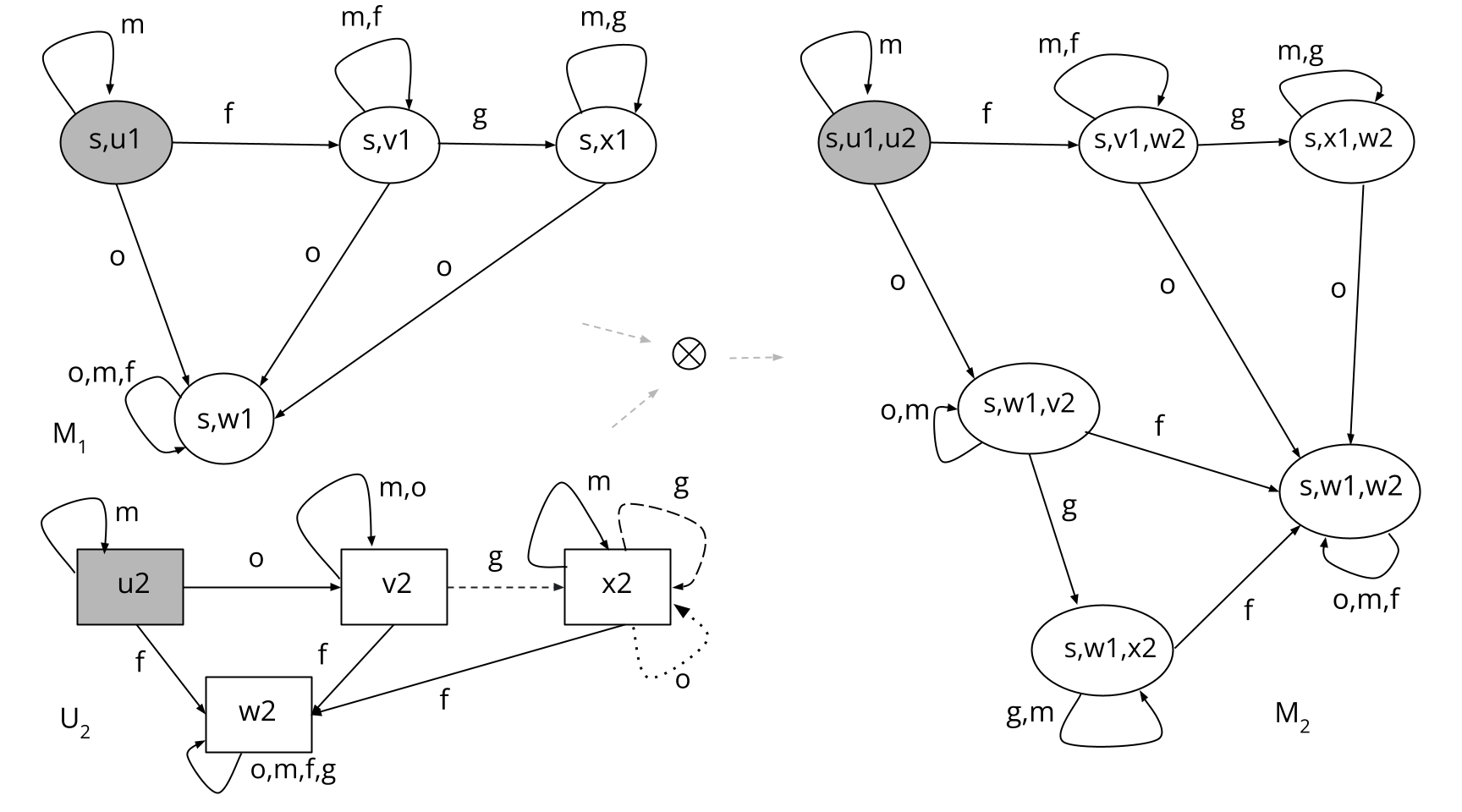}
    \caption{Mouse deceives owl that there is an owl-eating gruffalo}
    \label{fig:mousedcvsowl}
\end{figure}

\subsection{The owl is deceived}
\begin{example}
Next when the mouse runs into the owl, it makes a similar deceptive move again, 
at $u2$ the owl is deceived into believing in the agency of the gruffalo with fox being oblivious of the interaction;
at $v2$ the owl believes that the gruffalo believes in eating owls.
The model update is shown in Figure~\ref{fig:mousedcvsowl}.
\end{example}

Note that the owl starts believing in the agency of gruffalo but it does not know of the beliefs of the fox. Likewise the fox believes in the agency of gruffalo and it believes that the owl does not. 

In the book by \cite{donaldson1999gruffalo}, there is also the deception of a snake on similar lines,
which we do not get into. We have not introduced the snake as an agent at all.

\subsection{The gruffalo appears}\label{sec:appears}

Further in \emph{The Gruffalo}, Donaldson inventively mixes in magical realism.

\begin{figure}[!htb]
    \centering
    \includegraphics[width=0.9\textwidth]{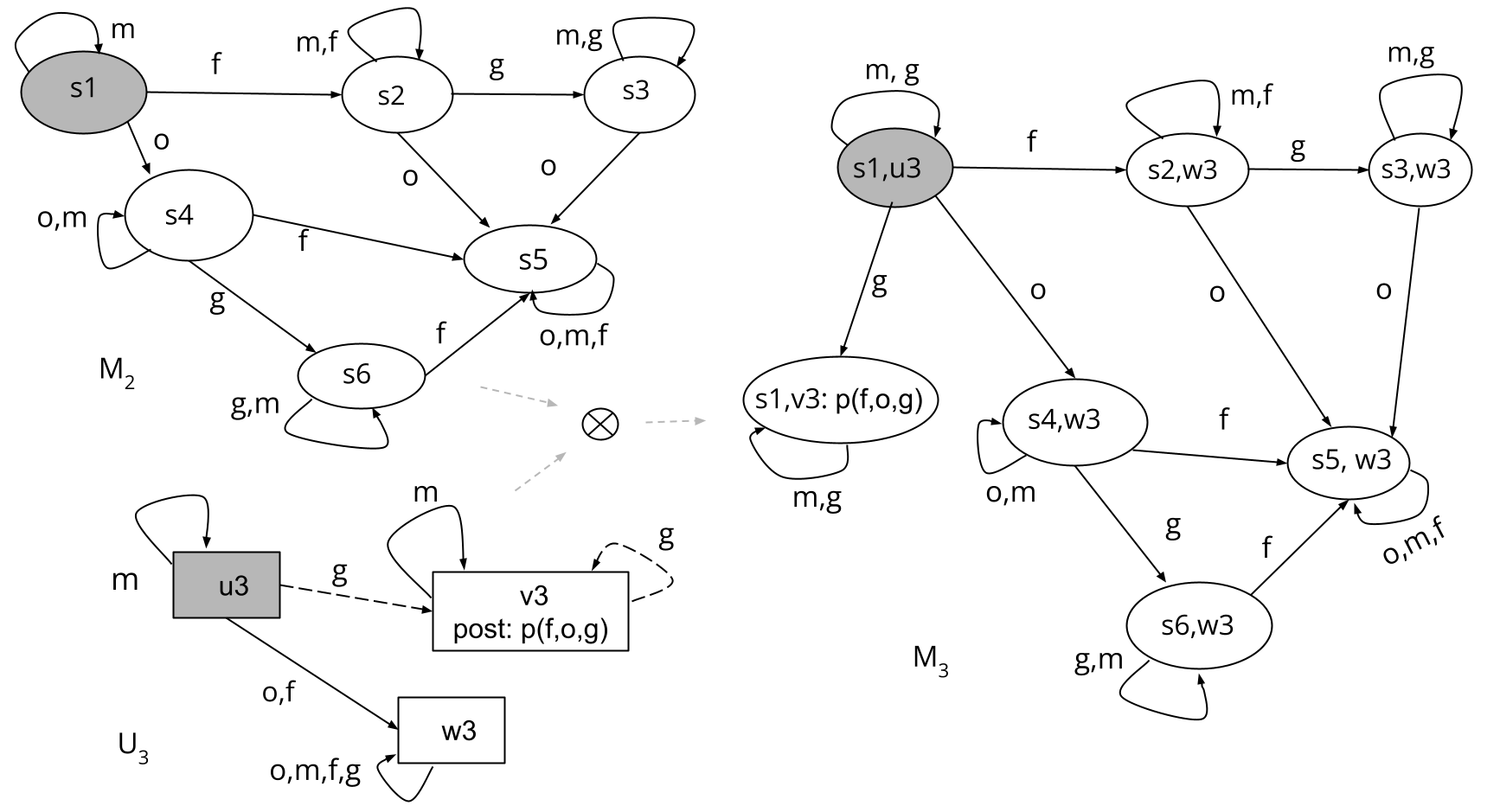}
    \caption{Gruffalo appears for mouse}
    \label{fig:realgruffalocombined}
\end{figure}

\begin{example}
\label{ex:realg}
The mouse runs into a `real' gruffalo, illustrated in Figure~\ref{fig:realgruffalocombined}. 
We model this as a private $g$-addition-update 
about $g$'s agency for mouse while the beliefs of fox and owl remain static. 
To save space, worlds $(s,u1,u2), (s,v1,w2), (s,x1,w2), (s,w1,v2),$ $ (s,w1,w2)$ and $(s,w1,x2)$ in model $M_2$ (Figure~\ref{fig:mousedcvsowl}) are renamed as $s1, s2, s3, s4, s5$ and $s6$ in Figure~\ref{fig:realgruffalocombined}.
\end{example}

We use a \emph{postcondition} fluent $p(f,o,g)$ at event $v3$, see update $U3$ in Figure~\ref{fig:realgruffalocombined}.
It stands for an implication
$implies(threat(f),threat(o),threat(g))$ which mouse communicates to gruffalo,
that if both fox and owl feel threatened by mouse,
then gruffalo is threatened by mouse. 
(It is not written as an EL formula because the fluents hold across different scenes of our story.)
This is one of the kinds of deception by lying considered by \cite{sakama}.


\begin{figure}[!htb]
    \centering
    \includegraphics[width=0.95\textwidth]{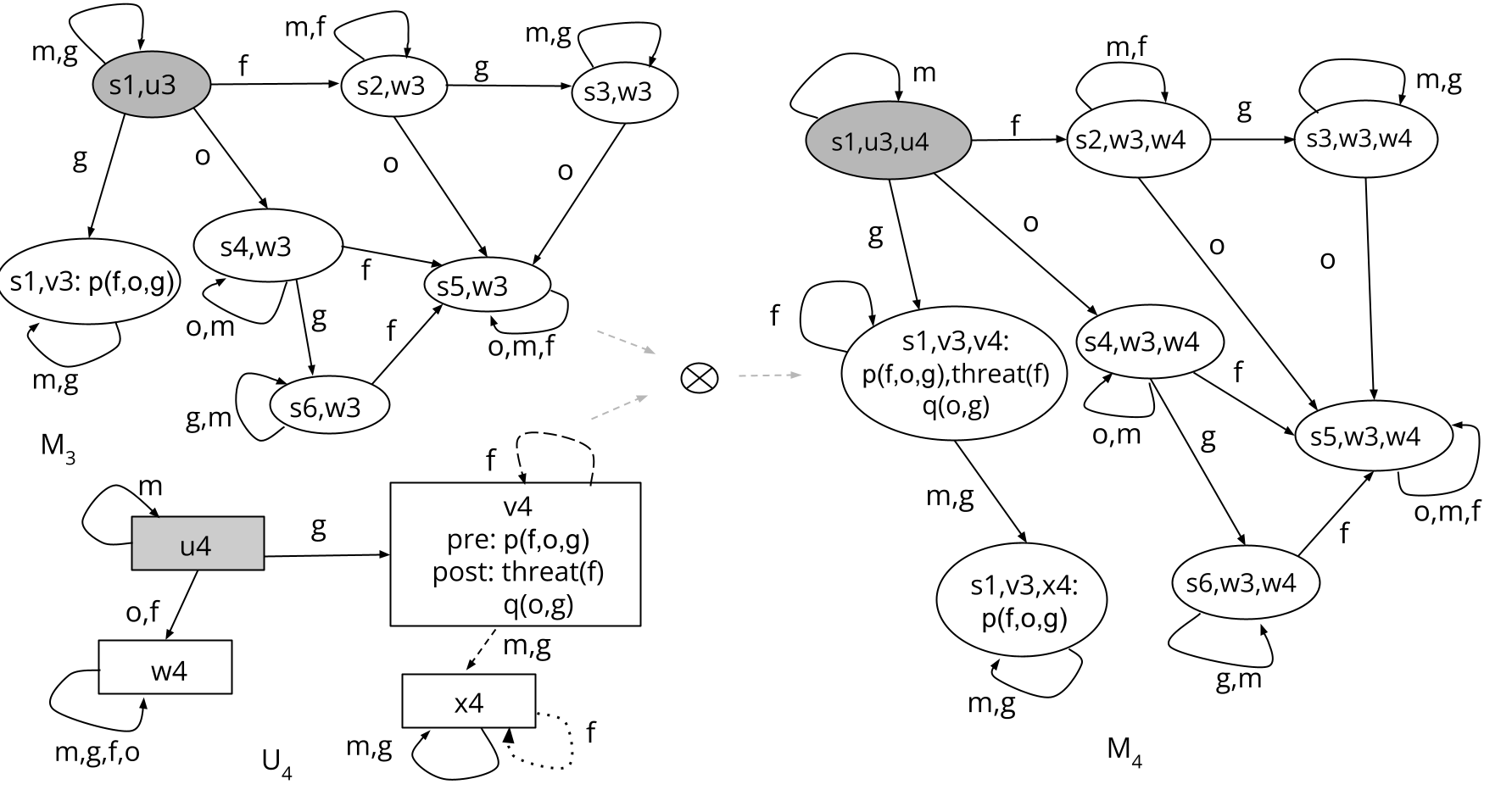}
    \caption{Mouse takes gruffalo to meet fox}
    \label{fig:gmeetsf}
\end{figure}

\subsection{Gruffalo meets fox}\label{sec:gmeetsf}

\begin{example}\label{ex:gmeetsf}
Further in the story, the mouse introduces  gruffalo to fox.
Mouse is an observer of event $u4$ at which $g$ observes $f$-addition at $v4$. 
Owl is oblivious at $w4$.
Fox's belief that it will be eaten away 
is unforgotten at $x4$
in the agent-update frame at the bottom of Figure \ref{fig:gmeetsf}. 
Hence the fox runs away as soon as it sees them.
Gruffalo's belief is that the fox, appearing at $v4$, ran away because it encountered them, 
these are the $m,g$-addition arrows at $v4$ which go to another world $x4$, where fox is ``eaten''. 
There is an ambiguity here between the $m,g$-arrows that the story exploits.
Fox believes $P_g \lnot P_f \true$.
Gruffalo chooses to believe $P_f P_m \lnot P_f \true$ at $(s1,v3,v4)$. This is represented as fluent $threat(f)$ at $v4$.
\end{example}

The fluent $p(f,o,g)$ of Section~\ref{sec:appears} is a \emph{precondition} to $v4$. 
In the presence of the fluent postcondition $threat(f)$,
at $v4$ there is a derived postcondition $q(o,g)$, a fluent standing for the implication $implies(threat(o),threat(g))$,
that if owl feels threatened by mouse,
then gruffalo is threatened by mouse.
\cite{singh2021mental} calls this an indirect or ramified effect of the action $U4$. 

\begin{figure}[!htb]
    \centering
    \includegraphics[width=0.9\textwidth]{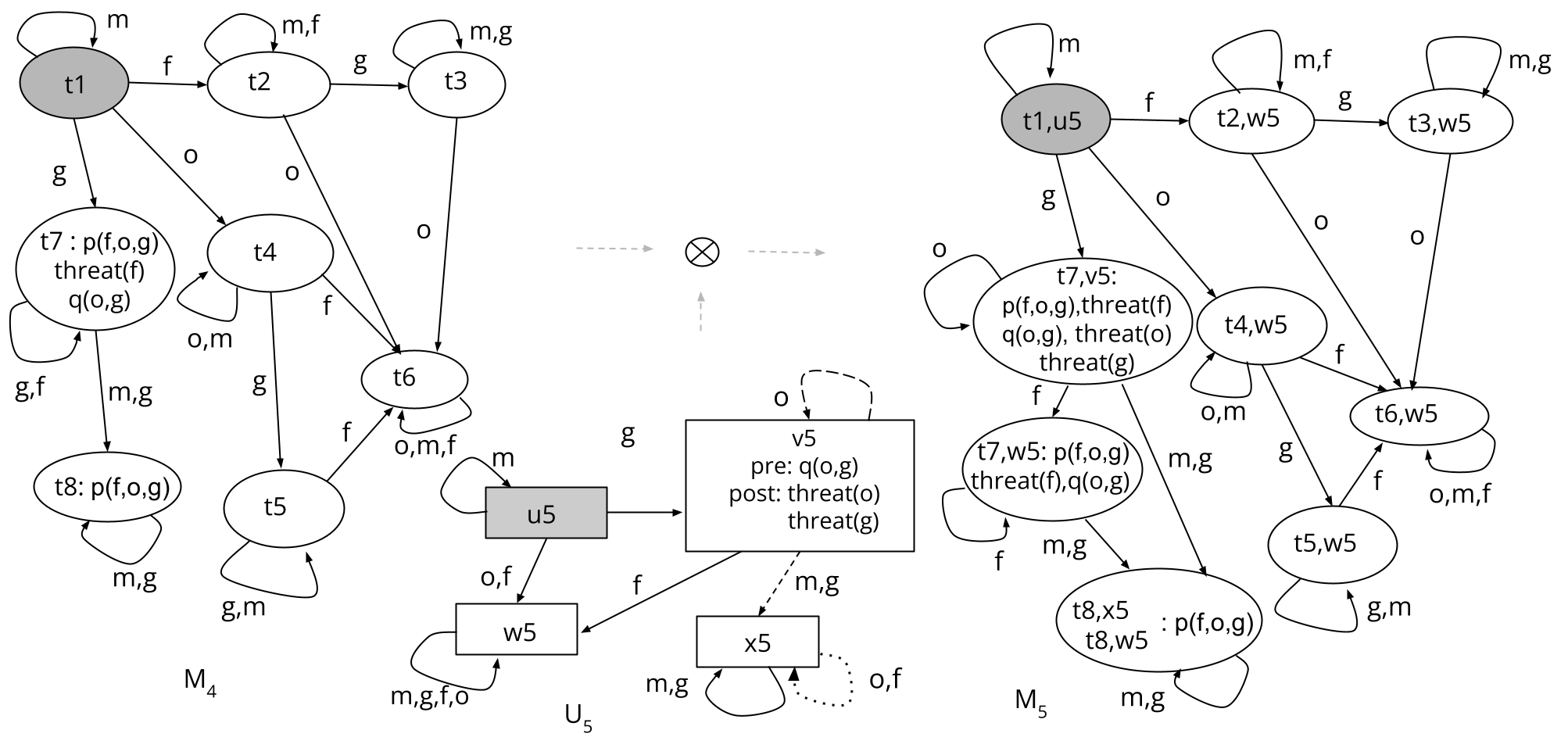}
    \caption{Mouse takes gruffalo to meet the owl; fox is oblivious}
    \label{fig:gmeetso}
\end{figure}

\subsection{Gruffalo meets owl}
\begin{example}
Next in the story, the mouse introduces gruffalo to owl 
as shown in Figure~\ref{fig:gmeetso}.
To save space, worlds $(s1,u3,u4)$, $(s2,w3,w4)$, $(s3,w3,w4)$, $(s4,w3,w4)$, $(s6,w3,w4)$, $(s5,w3,w4)$, $(s1,v3,v4)$ and $(s1,v3,x4)$ are renamed as $t1$, $t2$, $t3$, $t4$, $t5$, $t6$, $t7$ and $t8$, respectively. 
The owl runs away as soon as it sees them, since its belief that gruffalo eats owl is unforgotten.
Gruffalo's belief $P_o P_m \lnot P_o \true$ can be seen in the updated model at $(t7,v5)$.
This is represented as postcondition $threat(o)$.
\end{example}

Since the fluent $q(o,g)$ from Section~\ref{sec:gmeetsf} is a precondition at event $v5$, 
gruffalo infers a threat to itself,
so there is a derived postcondition $threat(g)$. 
The gruffalo runs away and the story has a happy ending. 


\section{Agent-update semantics}\label{sec:aul}
The logic $DEL$ is now redefined with agent-update frames being used as update actions.
(The language DEL does not change.)
The truth of formulas with update modalities is interpreted using sum-product update:

$(M,s) \models  \upd{u}\phi $~iff~ $ (M,s) \models pre(u)$ and  $(M * U, (s,u)) \models \phi$
 

The sublanguages of Definition \ref{def:logic} define the logics
$EL$, \delminustar\/ and \delstar, respectively, with this semantics.
We revisit the \emph{mouse-gruffalo} example one last time.

\begin{example}
Initially $M_0$ is defined such that $(M_0, s) \models P_m \top \wedge P_f \top \wedge P_o \top \wedge \neg P_g \top$. In Figure~\ref{fig:mousedcvsfox}, when the \textit{mouse} deceives the \textit{fox} about existence of a \textit{gruffalo}, keeping the \textit{owl} oblivious, $M_0 * U_1, (s,u1) \models \neg P_g \top \wedge P_f P_g \top \wedge B_m B_f P_g \top \wedge \neg P_o P_g \top.$ When the \textit{mouse} runs into the \textit{owl}, it deceives the \textit{owl} too but keeping the \textit{fox} oblivious this time, as shown in Figure~\ref{fig:mousedcvsowl}, and therefore in the updated model none of them know that the other is aware of \textit{gruffalo}'s agency: $M_0 * U_1 * U_2, (s,u1,u2) \models \neg P_g \top \wedge P_f P_g \top \wedge P_o P_g \top \wedge \neg P_f P_o P_g \top \wedge \neg P_o P_f P_g \top$. 

Eventually the \textit{gruffalo} appears in front of the \textit{mouse} leading to the update of the \textit{mouse}'s beliefs in the existence of \textit{gruffalo}: $M_0 * U_1 * U_2 * U_3, (s,u1,u2,u3) \models P_m P_g \top$. 
The cunning \textit{mouse} contrives a situation where the \textit{gruffalo} believes that \textit{fox} is eaten away (as illustrated in  Figure~\ref{fig:gmeetsf}), then 
according to the the information state of gruffalo: $M_0 * U_1 * U_2 * U_3 * 
U_4, (s,u1,u2,u3,u4) \models  P_g ~(p ~\wedge~ P_m \neg P_f \top)$.
But it can be seen from the model (Figure~\ref{fig:gmeetsf}) that:
$M_0 * U_1 * U_2 * U_3 * 
U_4, (s,u1,u2,u3,u4) \models P_f \top 
\wedge B_m P_f \top$.  
Similarly, after the Gruffalo is taken to the owl, its information state is updated such that: $M_0 * U_1 * U_2 * U_3 *
U_4 * U_5, (s,u1,u2,u3,u4,u5) \models  P_g ~(p ~\wedge~q~\wedge~ P_m (\neg P_f \top ~\wedge~ \neg P_o \top))$.
\end{example}

\begin{theorem}[Model checking]
Given a transitive pointed Kripke model, whether a \delminustar\/ formula holds can be checked in polynomial time,
and whether a \delstar\/ formula holds can be checked in polynomial space.
\end{theorem}

There is a lower bound of polynomial space for \delprod\/ model checking \citep{aucher2013complexity}, 
as well as for \delminusprod\/ model checking over equivalence relations \citep{dehaan2021complexity}.

\subsection{Proof system}\label{sec:extended}
The proof system for the logic \delminustar\/ with agent-update semantics gives axioms and inference rules to prove valid formulas.
There are 9 axioms and 3 standard inference rules below. 
In the expression $\upd{u}P_a \phi$, $(U,u)$ stands for an agent-update, and $a$ is any agent, which may or may not be involved in the update.

\begin{enumerate}
    \item\label{item:pc} all instantiations of propositional tautologies 
    \item $B_a (\phi \implies \psi) \implies (B_a \phi \implies B_a \psi)$
    \item\label{item:trans} $B_a \phi \implies  B_a B_a \phi$ 
    \item\label{item:distr} $[U,u](\phi \implies \psi) \implies ([U,u]\phi \implies [U,u]\psi)$
    \item\label{item:val} $\upd{u}p \Leftrightarrow (pre(u) \land ((p \land (post(u)(p)=no)) \lor (post(u)(p)=\true)))$
    \item\label{item:det} $\upd{u}\neg \phi \Leftrightarrow (pre(u) \land \neg \upd{u}\phi)$
    \item\label{item:and} $\upd{u}(\phi \lor \psi) \Leftrightarrow (\upd{u}\phi \lor \upd{u}\psi)$ 
    
\item\label{item:knowledgeaction} $\upd{u}P_a \phi \Leftrightarrow (pre(u) \wedge \bigvee_{v: u\obs{a}v,\neg u\del{a}v} P_a \upd{v}\phi)$,
for $a \notin  (Add(u) \setminus Del(u))$ 

\item\label{item:inherit}
$\upd{u}P_{a} \phi \Leftrightarrow 
(pre(u) \wedge {(}
\begin{array}{l}
\bigvee_{v: u\obs{a}v} P_a \upd{v}\phi ~\vee~
~\bigvee_{w: u\new{a}w} \upd{w}\phi ~\vee~
~\bigvee_{b: b \in Obs(u)} P_b \upd{u}\phi
\end{array}~{)}
)$,\\
for $a \in (Add(u) \setminus Del(u))$


\item\label{item:mp} From $\phi$ and $\phi \implies \psi$, infer $\psi$
\item From $\phi$, infer $B_{a} \phi$
\item\label{item:nec} From $\phi$, infer $[U,u] \phi$
\end{enumerate}

It follows from \cite[Theorem 8.54]{van2007dynamic} that an additional axiom
suffices for logic \delstar:

\begin{enumerate}\setcounter{enumi}{12}
\item $\updk\phi \Leftrightarrow 
(\dia{U_1,u_1}\phi \lor \dots \lor \dia{U_k,u_k}\phi)$
\end{enumerate}

Distributivity of bi-implication is provable using axioms (\ref{item:and}),(\ref{item:distr}),(\ref{item:pc}),(\ref{item:mp}),(\ref{item:nec}):
$[U,u](\phi \Leftrightarrow \psi) \implies ([U,u]\phi \Leftrightarrow [U,u]\psi)$ \cite[Exercise 1.4]{chellas}.
This is used in the completeness proof.
%

\subsection{Soundness and completeness of the proof system}
\label{sec:sound}
\begin{theorem}[Soundness and completeness]
\label{thm:sound}
The proof system given in Section \ref{sec:extended} is sound and complete for transitive Kripke models.
\end{theorem}

Axioms $1-3$ are for the base $EL$ defined on transitive frames. Axiom 4 distributes the update modality over implication.
Axiom \ref{item:val} says propositional valuation does not change during updates except according to the setting given by the postcondition. 
The satisfaction of the propositional formulas is derived. 
Changes will affect formulas with belief modalities.
Axiom \ref{item:det} says updates are deterministic, so handling a belief modality correctly also works to settle a formula with the negation of a belief modality.  
Axiom \ref{item:and} is the standard axiom for distribution of an operator over disjunction. 

Reduction of belief after update $(U,u)$ for agents $a$ such that $a$ is not being added or being deleted at $u$ is achieved through the \emph{unforgotten-belief} axiom \ref{item:knowledgeaction}, which extends the \emph{action-belief} axiom known in the literature \citep{baltag2004logics,van2007dynamic}.
It says that for an agent $a$ not being added in the update, its possible beliefs reduce to those that were reachable before the update by arrows which were not deleted during the update.
The next axiom is a generalisation.

\subsection{Decidability} 
The logic \delminusprod\/ is defined in the book of \cite{van2007dynamic}, 
its complexity is not pinned down there. 
The minus sign in the superscript specifies that unions of update frames are not available in the syntax.
Together with our generalization to transitive frames, this enables a better upper bound than found in the literature.

\begin{theorem}[Satisfiability]
There is a polynomial space algorithm to check satisfiability of a \delminustar\/ formula. There is a nondeterministic exponential time algorithm to check satisfiability of a \delstar\/ formula.
\end{theorem}

\begin{proof}
Valid equivalences transform a \delstar\/ formula to doxastic logic $EL$ on transitive frames. 
As shown in the completeness proof, there is a reduction algorithm to do this.
Since there is an algorithm to check satisfiability of doxastic logic on $K4$ frames \citep{bdvmodal},
there is an algorithm to check \delstar\/ satisfiability.

For the complexity of the algorithm,
note that an update formula $\upd{u} \phi$ 
contains within it a description of the agent-update frame.
The number of agents, the number of events in it, the observability relations are all included in the size of the frame, defined in Section~\ref{sec:background}.
In the length of an update formula we include the size of the update frame. 

Consider a representative world where the reduction equivalence 
is used as a rewrite rule going from left to right.
The number of disjuncts on the right hand side of an axiom
is linear in the number of events in the agent update model, as Definition \ref{def:observer} defines observers in terms of events.
Thus the right hand side is bounded by a polynomial in the representation of $(U,u)$ and at most twice the modal depth of the formula, allowing for a rewrite step to go from an update modality over a belief modality to a belief modality over an update modality, before doing a further rewrite step.

Thus a satisfiability algorithm can run on a tree with depth linear in the modal depth of the formula;
for example, going from a parent node with the left hand side of a reduction equivalence to a child node with one of the disjuncts on the right hand side.
The tree has number of nodes exponential in the length of the formula $\phi$.

Checking whether an agent relation holds for a node can be done 
in time polynomial in the formulas which appear along the path from the node to the root of the tree.
Agent relations are recomputed every time they are required, so the algorithm can take exponential time, but working on one path uses at most polynomial space.

Such a (satisfying) path can be guessed and verified by a nondeterministic algorithm which takes space polynomial in the length of the formula.
This gives an algorithm for satisfiability of \delminustar.

For \delstar\/ this does not suffice since the $\updk$ modality can refer to different branches of the satisfying tree \citep{aucher2013complexity}. 
The whole satisfying tree has to be guessed and verified, which takes a nondeterministic algorithm taking time exponential in the length of the formula.
\end{proof}

We illustrate the status of the model checking and satisfiability problems for the different logics in the picture. 
The upper bounds for \delminusprod\/ have been improved from the earlier picture. 
All bounds shown are tight, the upper bounds match the lower bounds, apart from the possibilities that PTIME=PSPACE and
PSPACE=NEXPTIME.


\begin{figure}[!htb]
    \centering
    \includegraphics[width=0.95\textwidth]{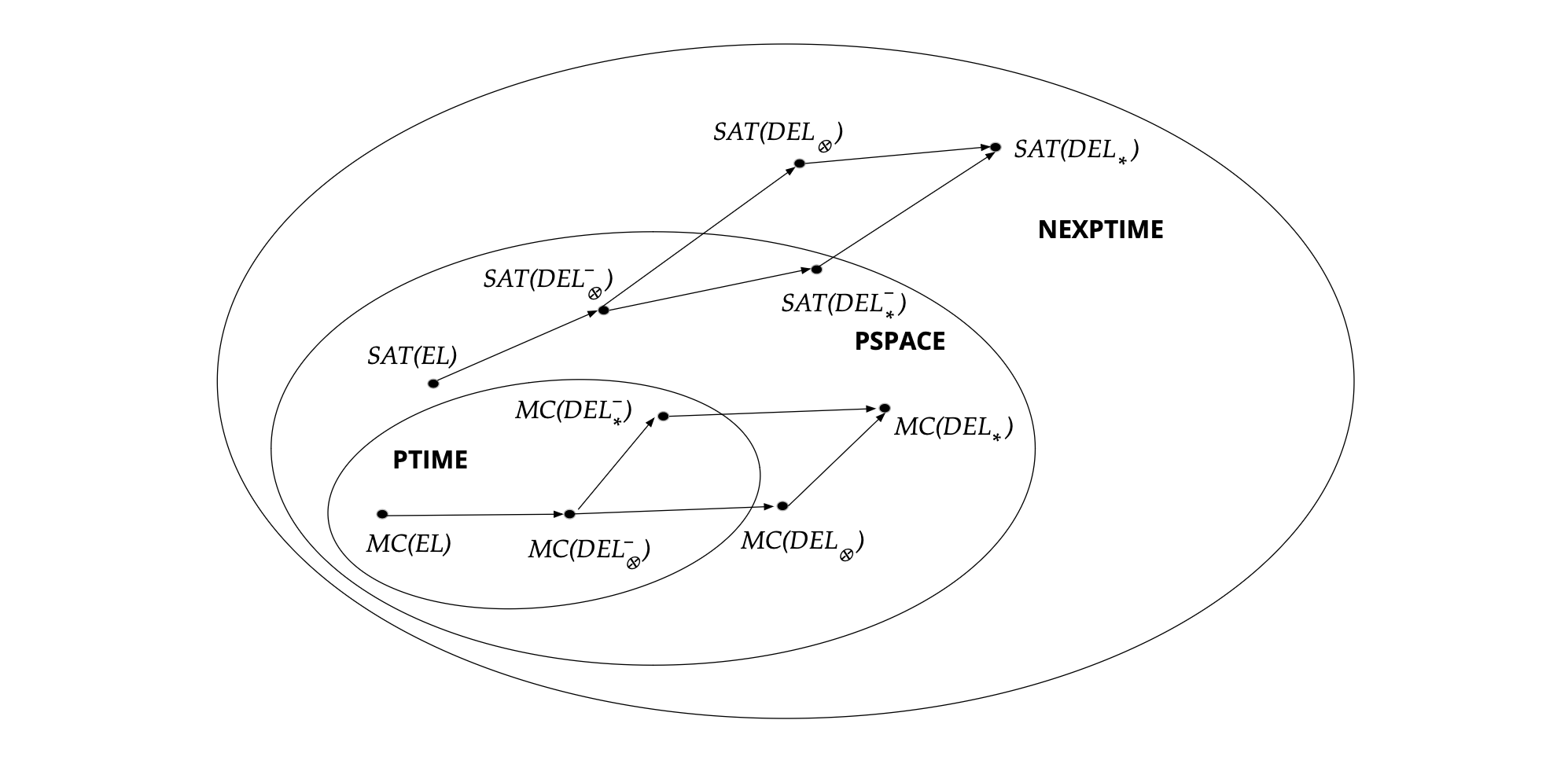}
    \caption{Upper bounds for $MC(\mathcal{L})$ and $SAT(\mathcal{L})$, revised}
    \label{fig:complexityclasses2}
\end{figure}


\section{Related work and discussion}
\label{sec:lit}
Approaches to model epistemic actions as changes on Kripke models broadly fall under two categories: state elimination and arrow elimination. 
For instance, \cite{plazapublic} models public announcements by eliminating states from the input model where the announced formula is false.
On the other hand, \cite{gerbrandy1997reasoning} model public announcements by eliminating arrows from the model such that only those worlds are reachable by the agents where the announced formula is true.

The action framework of \cite{baltag1998logic,baltag2004logics} is used to model not only epistemic actions but also ontic effects. We present our agent updates in the style of action frames.
The book by \cite{van2007dynamic} gives various dynamic epistemic logics.
Our language \delminusprod\/ is the language $\mathcal{L}^{-}_{K\otimes}$, as can be derived from \cite[Theorem 8.54]{van2007dynamic}.
Complexity of a dynamic logic avoiding nondeterministic choice was established in
\cite{halpern1983sdpdl}.
Complexity of model checking and satisfiability of \delprod\/ with nondeterministic choice
(the logic $\mathcal{L}_{K\otimes}$ of \cite{van2007dynamic}) 
was established by \cite{aucher2013complexity}.

Because our updates can add and delete agents, we can model interesting examples involving deception about agents.
The formal account of \cite{sakama} gives a classification of deceptions, of which we consider only simple forms in this paper,
which are called \emph{deception by lying} and \emph{intentional deception by lying}.
In our case the lying is more general, about existence of agents, and our semantics has additional features.
Modelling lies has been studied in dynamic epistemic logic \citep{sakama2010logical,Van2012lying,van2014dynamics}. 

\cite{sarkadi2019deception} model machine deception using ideas from communication theory 
and implement it in BDIA (Belief, Desires, Intentions and Action) agent architecture.
They require the precondition
that the deceived person is aware of the ascribed formula, 
which we do not require;
and that the deceiver has a `theory of mind'
about the deceived person's beliefs, desires, intentions, which we do not model.

Awareness of propositions has been studied in many papers, with agents becoming aware or unaware of them \citep{fagin1988belief,Van2010dynamics}. 
The papers \citep{Van2009awareness,Van2014semantics} specifically mention awareness of agents, 
using bisimulation quantifiers to model introduction of new propositions (which includes agent propositions). 
Our agent-update approach is completely different, it works at the level of semantic operations on action models.
A simpler form of quantification is used
by \cite{anantha2018bundled} to obtain polynomial space algorithms.

Independently of \cite{wang2022qfetmla} which has the same idea,
we model existence of agents at a world using presence of that agent's accessibility at the world.
Our results in Section~\ref{sec:aul} show that complexity is unaffected.

\cite{kpsf2019} study agent contexts in distributed systems with Byzantine faults.
The recent papers \citep{Van2021deadalive,glr2022} allow
removal of agents. These papers explore other interesting 
directions such as concurrency. 

\cite{amarel1971rep} suggested using the folk problem of 
missionaries and cannibals
to study planning problems in artificial intelligence. 
Smullyan's books, starting from \cite{smullyan}, have fascinating logic puzzles of various kinds. 
The book of \cite{prisonerlightbulb} is an inspiring account of modelling epistemic puzzles as stories. 
Books by \cite{woods1974logic,woods2018truth} explore the paradox that Sherlock Holmes lived in 221B Baker Street in the 19th century, and that he didn't since he didn't exist then.

We make extensive use of \emph{The Gruffalo} \citep{donaldson1999gruffalo} from children’s fiction to illustrate deception. 
We invite the reader to explore the book and its sequel \citep{donaldson2004child} for interesting ideas. 
\cite{logictovalue} points out that these stories (and others) have not been subject to much analysis in possible worlds theory.

\begin{remark}[Historical]
The story by \cite{donaldson1999gruffalo} is based on an Eastern folktale.
It is close to one in \cite{arabiannights}, a retelling of a 10th century Persian story in the collection \emph{One thousand and one nights} (also called Arabian nights).
Here a cat on the ground, after eating several hens, tells a cock on a tree, perched beyond its reach, that 
the King of Beasts has declared that all of them should love each other and proposes that they go for a stroll.
The cock says it can see that the hounds of the King of Beasts are coming to make others aware of the declaration.
The cat hurriedly leaves, and on being asked why, tells the cock that it is not sure that the hounds know of the declaration.
The essence of the story is the same, this time with the cock as deceiver and cat as deceived.
Modelling it would require a planning domain \citep{mcdermott2000}, showing the asymmetric location where the cock's visibility is higher than that of the cat.
Donaldson's story is simpler and easier to model in pure doxastic logic. 

It appears this Arabian nights story is in turn derived from a simpler one in the Buddhist Jataka tales, which go at least as far back as the 3rd century \citep{jatakatales}.
In the \emph{Kukkuta jataka}, engraved on a Buddhist \emph{stupa} at Bharhut in Madhya Pradesh, India, a cat declares its love for a cock and proposes to it.
The cock declares (at a higher plane) that there cannot be friendship or love between predator and prey.
There is no deception here.
Modelling the story would require quantified logic.
\end{remark}

\subsection{Discussion}
In this paper we have explored the idea of agent-addition and deletion in a group. 
We see our work as a starting point.
We can broadly classify the different kinds of agent-updates that are possible under nested levels of modalities,
as hinted at in Section~\ref{sec:obsdec}.
More refined ways of ascribing beliefs and nesting updates remain to be found.

Finding validities which express meaningful properties would be of interest.
\cite{wang2022qfetmla} has weaker forms of reflexivity,
which hold in models of epistemic term-modal logic.
Since our logic is doxastic, this particular property is not valid for our models. 
\cite{kpsf2019}'s ``hope'' axioms offer an interesting direction.
\cite{glr2022} work with symmetry properties intermediate between K4 and K45. 

Action frames in dynamic epistemic logic have traditionally been associated with `hard' update operations.
We have extended them with agent-update operations.
`Soft' operations such as \emph{upgrades} have also been studied for belief revision (for example, see the book of \citet{ben2010open}).
Interesting agent-update operations in such settings remain open. 
A start was made in \cite{singh2023two}.

The authors would like to thank Hans van Ditmarsch, Anantha Padmanabha, R.~Ramanujam and Yanjing Wang for discussions on an earlier version of this paper.

\bibliographystyle{apacite}
\bibliography{main}

\begin{thebibliography}{}
\bibitem [\protect \citeauthoryear {%
Smullyan%
}{%
Smullyan%
}{%
{\protect \APACyear {1978}}%
}]{%
smullyan}
\APACinsertmetastar {%
smullyan}%
\begin{APACrefauthors}%
Smullyan, R\BPBI M.%
\end{APACrefauthors}%
\unskip\
\newblock
\APACrefYear{1978}.
\newblock
\APACrefbtitle {What is the name of this book?} {What is the name of this
  book?}
\newblock
\APACaddressPublisher{}{Prentice-Hall}.
\PrintBackRefs{\CurrentBib}
\end{thebibliography}

\end{document}